\documentclass[english,12pt]{scrartcl}
\usepackage[T1]{fontenc}
\usepackage[utf8]{inputenc}
\usepackage{textcomp}
\usepackage{geometry}
\usepackage[active]{srcltx}
\usepackage{color}
\usepackage{babel}
\usepackage{booktabs}
\usepackage{amsmath}
\usepackage{amsthm}
\usepackage{amssymb}
\usepackage{bm}
\usepackage{textcomp, mathcomp}
\usepackage{array}
\usepackage{makecell}
\usepackage{multirow}
\usepackage{float}
\usepackage{comment}
\usepackage[numbers]{natbib}
\usepackage{diagbox}
\usepackage{lipsum}
\usepackage{scalefnt}
\usepackage[unicode=true,pdfusetitle,bookmarks=true,
bookmarksnumbered=true,bookmarksopen=true,
bookmarksopenlevel=1,breaklinks=true,pdfborder={0 0 1},backref=false,colorlinks=true]{hyperref}
\hypersetup{
 linkcolor=blue, citecolor=blue, urlcolor=blue, filecolor=blue,pdfpagelayout=OneColumn, pdfnewwindow=true,pdfstartview=XYZ, plainpages=false}

\makeatletter
\theoremstyle{plain}
\newtheorem{assumption}{\protect\assumptionname}
\theoremstyle{definition}
\newtheorem{defn}{\protect\definitionname}
\theoremstyle{definition}
\newtheorem{example}{\protect\examplename}
\theoremstyle{plain}
\newtheorem{lem}{\protect\lemmaname}
\theoremstyle{plain}
\newtheorem{prop}{\protect\propositionname}
\theoremstyle{remark}
\newtheorem{rem}{\protect\remarkname}
\theoremstyle{plain}

\usepackage{scrlayer}
\usepackage{lastpage}
\usepackage{amsmath}
\usepackage{amssymb}
\usepackage{graphicx}
\usepackage{xcolor}
\usepackage{fancybox} 
\usepackage{moreverb} 
\usepackage{listings} 
\usepackage{backref}
\usepackage{authblk}

\usepackage{ifpdf} % part of the hyperref bundle
\ifpdf % if pdflatex is used

 \IfFileExists{lmodern.sty}{\usepackage{lmodern}}{}

\fi % end if pdflatex is used

\let\myTOC\tableofcontents
\renewcommand\tableofcontents{%
  \pdfbookmark[1]{\contentsname}{}
  \myTOC
}

\def\LyX{\texorpdfstring{%
  L\kern-.1667em\lower.25em\hbox{Y}\kern-.125emX\@}
  {LyX}}

\@addtoreset{footnote}{section}

\renewcommand*{\backref}[1]{}
\renewcommand*{\backrefalt}[4]{%
   \ifcase #1 %no citation
    \or %cited on exactly one page
      (Cited on page~#2)%
   \else%cited on multiple pages
      (Cited on pages~#2)
    \fi} 

\usepackage{dsfont}
\usepackage{bbm}
\usepackage{mathtools,amsthm,amssymb,bbm}

\definecolor{blue}{HTML}{1F77B4}
\definecolor{orange}{HTML}{FF7F0E}
\definecolor{green}{HTML}{2CA02C}
\definecolor{red}{HTML}{D62728}
\definecolor{purple}{HTML}{9467BD}
\definecolor{brown}{HTML}{8C564B}
\definecolor{pink}{HTML}{E377C2}
\definecolor{grey}{HTML}{7F7F7F}
\definecolor{yellow}{HTML}{BCBD22}
\definecolor{cyan}{HTML}{17BECF}
\definecolor{turquoise}{HTML}{3FE0D0}

\usepackage[ruled,vlined]{algorithm2e}
\usepackage{xcolor}
\definecolor{algoColorKeyword}{named}{blue}
\definecolor{algoColorComment}{named}{olive}

\SetKwComment{Comment}{$\triangleright$\ }{}
\SetKwRepeat{Do}{do}{while}

\usepackage{enumitem}
\setlist{leftmargin=*, topsep=0.5em, parsep=0pt, itemsep=1em, labelindent=0pt, align=left}

\makeatother

\providecommand{\assumptionname}{Assumption}
\providecommand{\definitionname}{Definition}
\providecommand{\examplename}{Example}
\providecommand{\lemmaname}{Lemma}
\providecommand{\propositionname}{Proposition}
\providecommand{\remarkname}{Remark}
\providecommand{\theoremname}{Theorem}

\begin{document}
\title{Fast Gaussian process inference by exact Mat\'ern kernel decomposition}

\author[1]{Nicolas Langren\'e\thanks{Corresponding author, nicolaslangrene@uic.edu.cn}} 
\author[2]{Xavier Warin}
\author[2]{Pierre Gruet}

\affil[1]{\normalsize Guangdong Provincial/Zhuhai Key Laboratory of Interdisciplinary Research and Application for Data Science, Beijing Normal-Hong Kong Baptist University} 
\affil[2]{\normalsize EDF Lab, FiME (Laboratoire de Finance des March\'es de l'\'Energie)} 

\date{\today}

\maketitle
\begin{abstract}
To speed up Gaussian process inference, a number of fast kernel matrix-vector multiplication (MVM) approximation algorithms have been proposed over the years. In this paper, we establish an exact fast kernel MVM algorithm based on exact kernel decomposition into weighted empirical cumulative distribution functions, compatible with a class of kernels which includes multivariate Mat\'ern kernels with half-integer smoothness parameter. This algorithm uses a divide-and-conquer approach, during which sorting outputs are stored in a data structure. We also propose a new algorithm to take into account some linear fixed effects predictor function. Our numerical experiments confirm that our algorithm is very effective for low-dimensional Gaussian process inference problems with hundreds of thousands of data points. An implementation of our algorithm is available at \url{https://gitlab.com/warin/fastgaussiankernelregression.git}.\\

\end{abstract}

\section{Introduction}

\begin{comment}
To make notations consistent with the RFF paper, I will make the following
notational changes:
$$\tilde{K}(x_{i},x_{j})=K(x_{i}-x_{j})=k(\left\Vert x_{i}-x_{j}\right\Vert )$$
\begin{table}[H]
\begin{center}
\renewcommand{\arraystretch}{1.5} 
\begin{tabular}{|l|l|}
\hline 
old notation & new notation\tabularnewline
\hline 
\hline 
$K:\mathbb{R}^{d}\times\mathbb{R}^{d}\rightarrow\mathbb{R}$ & $\tilde{K}:\mathbb{R}^{d}\times\mathbb{R}^{d}\rightarrow\mathbb{R}$\tabularnewline
\hline 
$\tilde{K}:\mathbb{R}^{d}\rightarrow\mathbb{R}$ & $K:\mathbb{R}^{d}\rightarrow\mathbb{R}$\tabularnewline
\hline 
$\hat{K}:\mathbb{R}\rightarrow\mathbb{R}$ & $k:\mathbb{R}\rightarrow\mathbb{R}$\tabularnewline
\hline 
$\tilde{K}_{\varsigma,\ell}:\mathbb{R}^{d}\rightarrow\mathbb{R}$ & $k_{\varsigma,\ell}:\mathbb{R}^{d}\rightarrow\mathbb{R}$\tabularnewline
\hline 
$\underline{K}:\mathbb{R}^{d}\rightarrow\mathbb{R}$ & $k:\mathbb{R}^{d}\rightarrow\mathbb{R}$\tabularnewline
\hline 
\end{tabular}
\end{center}
\end{table}
\end{comment}

Consider $N$ input data points $\mathbf{x}_{1},\mathbf{x}_{2},\ldots,\mathbf{x}_{N}\in\mathbb{R}^{d}$
and their corresponding response values $y_{1},y_{2},\ldots,y_{N}\in\mathbb{R}$.
The Gaussian process regression model (see for example \citep{rasmussen2006gaussian}) is usually defined as follows:
\begin{eqnarray}
y_{i} & \sim &   f(\mathbf{x}_{i})+\varepsilon_{i},\qquad i=1,2,\ldots,N,\label{eq:gp_y}\\
f(\mathbf{x}) & \sim & \mathcal{GP}(m(\mathbf{x}),\tilde{K}(\mathbf{x},\mathbf{x}')),\label{eq:gp_f}
\end{eqnarray}
where   $\mathcal{GP}$ stands for a Gaussian process distribution with
prior mean $m:\mathbb{R}^{d}\rightarrow\mathbb{R}$, covariance kernel $\tilde{K}:\mathbb{R}^{d}\times\mathbb{R}^{d}\rightarrow\mathbb{R}$,
and the noise terms $\varepsilon_{i}\sim\mathcal{N}(0,\sigma^{2})$
are independent and identically distributed with known variance $\sigma^{2}>0$.
In this article we suppose that the  mean
function is either affine, so that
\begin{align}
   m(\mathbf{x}) =  \tilde \beta + \mathbf{\beta}^{\top} \mathbf{x},
   \label{eq:BetaLin}
\end{align}
with $\tilde \beta \in \mathbb{R}$, $\beta \in \mathbb{R}^{d}$, or can be expanded onto a set of basis functions $h_1, \ldots, h_L$ so  that
\begin{align}
    m(\mathbf{x}) = \mathbf{\beta}^{\top} \mathbf{h}(\mathbf{x}),
     \label{eq:BetaFunc}
\end{align}
 with $\mathbf{h}(\mathbf{x})= \{ h_1(\mathbf{x}), \ldots, h_L(\mathbf{x}) \}$. The marginal posterior distribution of $f$ at an
evaluation point $\mathbf{z}\in\mathbb{R}^{d}$ is Gaussian with mean
and variance given by
\begin{align}
\mathbb{E}[f(\mathbf{z})|\mathbf{x},\mathbf{y}] & =\mathbf{k}_{\mathbf{z}}^{\top}(\mathbf{K}+\sigma^{2}\mathbf{I})^{-1}(\mathbf{y} -  \mathbf{m}(\mathbf{x})) + m(\mathbf{z}),\label{eq:gp_mean}\\
\mathbb{V}\hspace{-0.08em}\mathrm{ar}[f(\mathbf{z})|\mathbf{x},\mathbf{y}] & =\tilde{K}(\mathbf{z},\mathbf{z})-\mathbf{k}_{\mathbf{z}}^{\top}(\mathbf{K}+\sigma^{2}\mathbf{I})^{-1}\mathbf{k}_{z},\label{eq:gp_variance}
\end{align}
where $\mathbf{x}=\{\mathbf{x}_{1},\mathbf{x}_{2},\ldots,\mathbf{x}_{N}\}$,
and $\mathbf{y}=\{y_{1},y_{2},\ldots,y_{N}\}$ form the input dataset,
$\mathbf{m}(\mathbf{x})= \{ m(\mathbf{x}_{1}),m(\mathbf{x}_{2}),\ldots, m(\mathbf{x}_{N})\} \in\mathbb{R}^{N}$,
$\mathbf{k}_{\mathbf{z}}=[\tilde{K}(\mathbf{x}_{i},\mathbf{z})]_{1\leq i\leq N}\in\mathbb{R}^{N}$
is the covariance between the training dataset and the evaluation
point $\mathbf{z}$, $\mathbf{K}=[\tilde{K}(\mathbf{x}_{i},\mathbf{x}_{j})]_{1\leq i,j\leq N}\in\mathbb{R}^{N\times N}$
is the covariance matrix evaluated at the input dataset, and $\mathbf{I}\in\mathbb{R}^{N\times N}$
is the identity matrix. Once $m$ has been estimated, the simple translation $ (\mathbf{y}- \mathbf{m} (\mathbf{x}), f(\mathbf{z})-m(\mathbf{z})) \longrightarrow ( \mathbf{y} , f(\mathbf{z}))$ reduces the problem to solving \eqref{eq:gp_mean} with $m =0$, so that
\begin{align}
\mathbb{E}[f(\mathbf{z})|\mathbf{x},\mathbf{y}] & =\mathbf{k}_{\mathbf{z}}^{\top}(\mathbf{K}+\sigma^{2}\mathbf{I})^{-1}\mathbf{y}. \label{eq:gp_meanS}
\end{align}
We make this simplifying centering assumption for the rest of the introduction.

In practice, the covariance kernel $\tilde{K}=\tilde{K}_{\bm{\theta}}$ depends on
some hyperparameters $\bm{\theta}$, which are to be calibrated on
the dataset $\{\mathbf{x},\mathbf{y}\}$ of interest. To do so, the
most common approach is to maximize the log-marginal likelihood
\begin{equation}
\mathcal{L}(\bm{\theta})=\log p(\mathbf{y}\left|\mathbf{x},\bm{\theta}\right.)=-\frac{1}{2}\mathbf{y}^{\top}(\mathbf{K}+\sigma^{2}\mathbf{I})^{-1}\mathbf{y}-\frac{1}{2}\mathrm{\log}(\det(\mathbf{K}+\sigma^{2}\mathbf{I}))-\frac{N}{2}\log(2\pi)\,.\label{eq:gp_log_likelihood}
\end{equation}
This can be achieved numerically using for example an iterative conjugate
gradient ascent algorithm. The gradient of the log-marginal likelihood
$\mathcal{L}$ with respect to $\bm{\theta}$ is given by
\begin{equation}
\frac{\partial\mathcal{L}}{\partial\bm{\theta}}(\bm{\theta})=\frac{1}{2}\mathbf{y}^{\top}(\mathbf{K}+\sigma^{2}\mathbf{I})^{-1}\frac{\partial\mathbf{K}}{\partial\bm{\theta}}(\mathbf{K}+\sigma^{2}\mathbf{I})^{-1}\mathbf{y}-\frac{1}{2}\mathrm{tr\!}\left((\mathbf{K}+\sigma^{2}\mathbf{I})^{-1}\frac{\partial\mathbf{K}}{\partial\bm{\theta}}\right)\,.\label{eq:gp_log_likelihood_derivative}
\end{equation}
It is apparent that the implementation of equations \eqref{eq:gp_meanS}-\eqref{eq:gp_log_likelihood}-\eqref{eq:gp_log_likelihood_derivative}
is a computational challenge. Indeed, a direct computation of the
inverse matrix $(\mathbf{K}+\sigma^{2}\mathbf{I})^{-1}$ requires
$\mathcal{O}(N^{3})$ operations (using Gauss-Jordan elimination for
example) and an $\mathcal{O}(N^{2})$ memory size, which is prohibitive
for large-scale problems. This can be slightly improved to about $\mathcal{O}(N^{2.8074})$
using Strassen's algorithm (see \citet{petkovic2009matrix}), but
this does very little to alleviate the computational burden of \eqref{eq:gp_meanS}-\eqref{eq:gp_log_likelihood}-\eqref{eq:gp_log_likelihood_derivative}.

Due to the growing importance and popularity of Gaussian process inference
in various fields (geostatistics, machine learning, engineering and
others, see for example \citet{gramacy2020surrogates} or \citet{binois2022survey}),
a vast literature on how to speed up the computation of \eqref{eq:gp_meanS}-\eqref{eq:gp_log_likelihood}-\eqref{eq:gp_log_likelihood_derivative}
has been developed over the years. Surveys and comparisons of existing
methods include \citet{sun2012geostatistics}, \citet{chalupka2013framework},
\citet{bradley2016comparison} \citet{heaton2019case}, \citet{liu2019understanding},
\citet{liu2020gaussian}, \citet{martinsson2020randomized}, \citet{huang2021competition}
and \citet{abdulah2022second}.

One broad category of approaches that stands out in the above articles
proceeds by replacing the covariance kernel matrix $\mathbf{K}$ by a sparse
/ low rank / local approximation, see for example \citet{snelson2005sparse},
\citet{titsias2009variational}, \citet{gramacy2015local}, \citet{katzfuss2021vecchia}
and \citet{maddox2021conditioning}. A different approach is to utilize
increasing computational power via parallel computing and/or GPU acceleration,
see for example \citet{deisenroth2015distributed}, \citet{matthews2017gpflow},
\citet{gardner2018gpytorch}, \citet{nguyen2019exact}, \citet{zhang2019parallel},
\citet{meanti2020kernel} \citet{charlier2021kernel}, \citet{hu2022giga},
\citet{abdulah2023large} and \citet{noack2023exact}.

An important remark is that equation \eqref{eq:gp_meanS}
can be rewritten without matrix inverse by introducing the solution
of the corresponding linear system. More specifically
\begin{align}
\mathbb{E}[f(\mathbf{z})|\mathbf{x},\mathbf{y}] & =\mathbf{k}_{\mathbf{z}}^{\top}\bm{\alpha}\label{eq:gp_mean_2}
\end{align}
where the vector $\bm{\alpha}=\{\alpha_{1},\alpha_{2},\ldots,\alpha_{N}\}\in\mathbb{R}^{N}$ is the unique solution of the $N\times N$ symmetric linear systems
\begin{align}
(\mathbf{K}+\sigma^{2}\mathbf{I})\bm{\alpha} & =\mathbf{y}\label{eq:gp_mean_system}
\end{align}
From here, a common starting point to decrease the $\mathcal{O}(N^{3})$
computational cost is to estimate the solution of the linear system
\eqref{eq:gp_mean_system} numerically
by iterative methods such as conjugate gradient (CG), instead of computing
the exact solution of \eqref{eq:gp_mean_system}
in $\mathcal{O}(N^{3})$ operations. Let $t_{\epsilon}$ be the number
of iterations of the conjugate gradient routine to achieve a predefined
accuracy level $\varepsilon>0$. This conjugate gradient approach
reduces the computational cost of Gaussian process prediction to $\mathcal{O}(t_{\epsilon}N^{2})$
and can be accelerated with appropriate preconditioning, see \citet{stein2012difference},
\citet{chen2013preconditioning}, \citet{cutajar2016preconditioning},
\citet{rudi2017falkon}, \citet{gardner2018gpytorch}, \citet{wang2019exact},
\citet{maddox2021iterative}, \citet{wenger2022preconditioning} and
\citet{zhang2022conjugate}, such that $t_{\varepsilon}\ll N$
in practice.

At this stage, the core computational step in each conjugate-gradient
iteration is the multiplication of the kernel matrix $\mathbf{K}$
with a vector, for an $\mathcal{O}(N^{2})$ computational cost. Further
speed-up therefore requires to accelerate such matrix-vector product,
by taking advantage of the fact that the matrix $\mathbf{K}$ is defined
by a covariance function. This generally requires additional approximations
to be made. Examples of fast matrix-vector multiplication (MVM) in
this context include \citet{quinonero2007approximation}, \citet{gardner2018product},
\citet{wang2019exact}, \citet{hu2022giga}, \citet{ryan2022fast}
and \citet{chen2023parallel}.

In this article, we introduce a new fast and exact kernel matrix-vector multiplication
algorithm, based on a recently established decomposition of kernel
functions into weighted sums of empirical cumulative distribution
function \citep{langrene2021fast}. We show that this computational
technique makes it possible to compute exact, dense kernel MVM for
an $\mathcal{O}(N\times(\log N)^{\max(1,d-1)})$ computational cost
and $\mathcal{O}(N)$ memory cost where $N$ is the number of input
data points in $\mathbb{R}^{d}$, for a class of stationary kernels
which includes the popular Mat\'ern covariance kernels.
We further improve the implementation of the fast CDF algorithm from \cite{langrene2021fast} by introducing a data structure to precalculate all the sorts needed. This reduces the computational cost to $\mathcal{O}(N\times(\log N)^{d-1})$, and greatly reduces the multiplicative computational constant.
%If a data structure is used to precalculate all the sorts needed, this cost decreases to $\mathcal{O}(N\times(\log N)^{d-1})$.
Compared
to the $\mathcal{O}(N^{2})$ computational and memory costs of direct
MVM, this kernel CDF decomposition approach is very competitive on
problems with large $N$ (hundreds of thousands and above) and small $d$ (say,
$d<5$), which covers many important application areas such as geostatistics.
The underlying computational speed-up relies on the fast computation
of multivariate empirical cumulative distribution functions by a divide-and-conquer
algorithm as described in \citet{bentley1980divide}, \cite{bouchard2012monte}, which, like a number
of fast GP techniques mentioned above, could also be further accelerated
using parallel processing.

We implemented our fast kernel matrix-vector multiplication algorithm,
and integrated it into an iterative solver based on conjugate gradients,
Lanczos tridiagonalizations, stochastic trace estimation and appropriate
preconditioning as detailed in Appendix~\ref{sec:likelihood_optimization} in order to compute fast Gaussian process predictions
\eqref{eq:gp_mean} and fast estimation of Gaussian process parameters
by maximum log-likelihood \eqref{eq:gp_log_likelihood}. Our numerical
tests, focusing on Mat\'ern covariance matrices, confirm the speed
and accuracy of the proposed algorithms. The rest of the paper is
organized as follows. Section~\ref{sec:kernel_mvm} describes our
proposed fast kernel matrix-vector multiplication algorithms.  Section~\ref{sec:numericalResults}
details our numerical results  assuming either  that $m=0$ or  that $m(\mathbf{x}) = \mathbf{\beta}^{\top} \mathbf{x}$ where we propose an improved algorithm for estimating $\mathbf{\beta}$, and Section \ref{sec:conclusionsec} concludes.

%Further details about likelihood optimization, Mat\'ern covariance
%functions, their decomposition into cumulative distribution functions, and variance reduction of random projection estimators
%are available in Appendices \ref{sec:likelihood_optimization}, \ref{sec:matern}, \ref{sec:cdf_decomposition_matern}, and \ref{sec:variance-reduction} respectively.

\section{Fast kernel matrix-vector multiplication\label{sec:kernel_mvm}}

Consider a matrix-vector multiplication (MVM) where the matrix is
a covariance matrix $\mathbf{K}:=\left[\tilde{K}(\mathbf{x}_{i},\mathbf{x}_{j})\right]_{1\leq i,j\leq N}\in\mathbb{R}^{N\times N}$
defined by a kernel $\tilde{K}:\mathbb{R}^{d}\times\mathbb{R}^{d}\rightarrow\mathbb{R}$ over a dataset $\mathbf{x}_{i}\in\mathbb{R}^{d}, i=1,\ldots,N$.
The product between the matrix $\mathbf{K}$ and the vector $\mathbf{y}:=\left[y_{i}\right]_{1\leq i\leq N}\in\mathbb{R}^{N}$
is given by equation \eqref{eq:Ky} below:

\begingroup\addtolength{\jot}{1em}

\begin{align}
 & \mathbf{K}:=\left[\tilde{K}(\mathbf{x}_{i},\mathbf{x}_{j})\right]_{1\leq i,j\leq N}:=\left[\begin{array}{ccccc}
\tilde{K}(\mathbf{x}_{1},\mathbf{x}_{1}) & \tilde{K}(\mathbf{x}_{1},\mathbf{x}_{2}) & \tilde{K}(\mathbf{x}_{1},\mathbf{x}_{3}) & \cdots & \tilde{K}(\mathbf{x}_{1},\mathbf{x}_{N})\\
\tilde{K}(\mathbf{x}_{2},\mathbf{x}_{1}) & \tilde{K}(\mathbf{x}_{2},\mathbf{x}_{2}) & \tilde{K}(\mathbf{x}_{2},\mathbf{x}_{3}) & \cdots & \tilde{K}(\mathbf{x}_{2},\mathbf{x}_{N})\\
\tilde{K}(\mathbf{x}_{3},\mathbf{x}_{1}) & \tilde{K}(\mathbf{x}_{3},\mathbf{x}_{2}) & \tilde{K}(\mathbf{x}_{3},\mathbf{x}_{3}) & \cdots & \tilde{K}(\mathbf{x}_{3},\mathbf{x}_{N})\\
\vdots & \vdots & \vdots & \ddots & \vdots\\
\tilde{K}(\mathbf{x}_{N},\mathbf{x}_{1}) & \tilde{K}(\mathbf{x}_{N},\mathbf{x}_{2}) & \tilde{K}(\mathbf{x}_{N},\mathbf{x}_{3}) & \cdots & \tilde{K}(\mathbf{x}_{N},\mathbf{x}_{N})
\end{array}\right]\label{eq:covariance_matrix}\\
 & \mathbf{y}:=\left[y_{i}\right]_{1\leq i\leq N}:=\left[\begin{array}{c}
y_{1}\\
y_{2}\\
y_{3}\\
\vdots\\
y_{N}
\end{array}\right]\hspace{20mm}\mathbf{K}\mathbf{y}=\left[\begin{array}{c}
\sum_{i=1}^{N}y_{i}\tilde{K}(\mathbf{x}_{1},\mathbf{x}_{i})\\
\sum_{i=1}^{N}y_{i}\tilde{K}(\mathbf{x}_{2},\mathbf{x}_{i})\\
\sum_{i=1}^{N}y_{i}\tilde{K}(\mathbf{x}_{3},\mathbf{x}_{i})\\
\vdots\\
\sum_{i=1}^{N}y_{i}\tilde{K}(\mathbf{x}_{N},\mathbf{x}_{i})
\end{array}\right]\label{eq:Ky}
\end{align}

\endgroup

Equation \eqref{eq:Ky} shows that computing the matrix-vector product
$\mathbf{K}\mathbf{y}$ amounts to computing the weighted kernel density
estimators (KDE, up to a multiplicative constant)
\begin{equation}
\sum_{i=1}^{N}y_{i}\tilde{K}(\mathbf{x}_{i},\mathbf{z})\,\,\mathrm{for\,\,\mathrm{all\,\,}}\mathbf{z}\in\left\{ \mathbf{x}_{1},\mathbf{x}_{2},\ldots,\mathbf{x}_{N}\right\} \label{eq:kde}
\end{equation}
A direct implementation of this matrix-vector product \eqref{eq:Ky}-\eqref{eq:kde}
requires $\mathcal{O}(N^{2})$ operations and $\mathcal{O}(N^{2})$ memory storage. 

Recently, \citet{langrene2021fast} introduced an exact kernel decomposition
approach into weighted sums of empirical cumulative distribution functions
(CDFs), which was used for fast kernel density estimation and fast
kernel regression. For comprehensiveness, we recall how this method
works on univariate data.
Then we explain how to extend the methodology to the multivariate case  and refer to \citet{langrene2021fast} and
its supplementary material for the detailed multivariate implementation.

We begin with two classical definitions:
\begin{itemize}
    \item A kernel $\tilde{K}$ is said to be \textit{shift-invariant} (a.k.a.
translation-invariant, radially-symmetric, or stationary) if for any $\mathbf{x}_{i}\in\mathbb{R}^{d}$
and $\mathbf{x}_{j}\in\mathbb{R}^{d}$, $\tilde{K}(\mathbf{x}_{i},\mathbf{x}_{j})$ only depends on $\mathbf{x}_{i}$ and $\mathbf{x}_{j}$ through the
difference $\mathbf{x}_{i}-\mathbf{x}_{j}$:
$$\tilde{K}(\mathbf{x}_{i},\mathbf{x}_{j})= K(\mathbf{x}_{i}-\mathbf{x}_{j})$$
where $K:\mathbb{R}^d\rightarrow\mathbb{R}$. Moreover, the shift-invariant
kernel $K:\mathbb{R}^{d}\rightarrow\mathbb{R}$ is said to be \textit{isotropic} \citep{genton2001kernels} if it only depends on $\mathbf{x}_{i}$ and $\mathbf{x}_{j}$ through
the Euclidean norm $\left\Vert \mathbf{x}_{i}-\mathbf{x}_{j}\right\Vert $
of the difference $\mathbf{x}_{i}-\mathbf{x}_{j}$:
\begin{equation}
    K(\mathbf{x}_{i}-\mathbf{x}_{j})= k(\left\Vert \mathbf{x}_{i}-\mathbf{x}_{j}\right\Vert).\label{eq:isotropic}
\end{equation}
where $k:\mathbb{R}\rightarrow\mathbb{R}$.
\item A kernel $\tilde{K}$ is positive definite if
for any $N\geq1$, $(\mathbf{x}_{1},\ldots,\mathbf{x}_{N})\in\mathbb{R}^{d\times N}$
and any $(y_{1},\ldots,y_{N})\in\mathbb{R}^{N}$,
\begin{equation}
\sum_{i=1}^{N}\sum_{j=1}^{N}y_{i}y_{j}\tilde{K}(\mathbf{x}_{i},\mathbf{x}_{j})\geq0.\label{eq:positive_definite}
\end{equation}
\end{itemize}
In the rest of the paper, we are going to solely focus on continuous,
positive definite, radially-symmetric  kernels. In this situation, the kernel function $k:\mathbb{R}\rightarrow\mathbb{R}$ in equation~\eqref{eq:isotropic} is often parametrized by an outputscale parameter $\varsigma>0$ and a lengthscale parameter $\ell>0$ as
\begin{equation}
k(u)=k_{\varsigma,\ell}(u):=\varsigma^{2}k(u/\ell),\label{eq:scaled_kernel}
\end{equation}
and the set of parameters is denoted as $\theta=(\varsigma,\ell)$.

\subsection{The univariate case\label{subsec:univariate-case}}

Let $K$ be a univariate kernel satisfying the following assumption.
\begin{assumption}
\label{assu:kernel_decomposition_univariate}The shift-invariant,
isotropic kernel $K(u)=k(\left|u\right|)$ admits the following decomposition:
\begin{equation}
k(u-v)=\sum_{p=1}^{P}\varphi_{1,p}(u)\varphi_{2,p}(v),\ \forall u\in\mathbb{R},v\in\mathbb{R}\label{eq:kernel_decomposition_assumption}
\end{equation}
for some functions $\varphi_{1,p}:\mathbb{R}\rightarrow\mathbb{R}$
and $\varphi_{2,p}:\mathbb{R}\rightarrow\mathbb{R}$.
\end{assumption}
Under Assumption~\ref{assu:kernel_decomposition_univariate}, the
matrix-vector product \eqref{eq:Ky}-\eqref{eq:kde} can be decomposed as 
\begin{align}
 & \sum_{i=1}^{N}y_{i}K(x_{i}-z)=\sum_{i=1}^{N}y_{i}k(\left|x_{i}-z\right|)\nonumber \\
 & =\sum_{i=1}^{N}y_{i}k(z-x_{i})\mathbbm{1}\{x_{i}\leq z\}+\sum_{i=1}^{N}y_{i}k(x_{i}-z)\mathbbm{1}\{x_{i}>z\}\nonumber \\
 & =\sum_{i=1}^{N}y_{i}\sum_{p=1}^{P}\varphi_{1,p}(z)\varphi_{2,p}(x_{i})\mathbbm{1}\{x_{i}\leq z\}+\sum_{i=1}^{N}y_{i}\sum_{p=1}^{P}\varphi_{1,p}(-z)\varphi_{2,p}(-x_{i})\mathbbm{1}\{x_{i}>z\}\nonumber \\
 & =\varphi_{1,p}(z)\sum_{p=1}^{P}\left(\sum_{i=1}^{N}y_{i}\varphi_{2,p}(x_{i})\mathbbm{1}\{x_{i}\leq z\}\right)+\varphi_{1,p}(-z)\sum_{p=1}^{P}\left(\sum_{i=1}^{N}y_{i}\varphi_{2,p}(-x_{i})\mathbbm{1}\{x_{i}>z\}\right)\label{eq:fast_exact_decomposition_1d}
\end{align}
In other words, equation \eqref{eq:kde} is equal to a weighted sum
of $P$ weighted empirical cumulative distribution functions (CDF)
$\sum_{i=1}^{N}y_{i}\varphi_{2,p}(x_{i})\mathbbm{1}\{x_{i}\leq z\}$
and $P$ weighted empirical survival functions $\sum_{i=1}^{N}y_{i}\varphi_{2,p}(-x_{i})\mathbbm{1}\{x_{i}>z\}$.
This decomposition \eqref{eq:fast_exact_decomposition_1d} can be efficiently
computed for all $z\in\left\{ x_{1},x_{2},\ldots,x_{N}\right\} $
by sorting the dataset in increasing order and computing the CDFs
in increasing order from $z=x_{1}$ to $z=x_{N}$. This is efficient
because computing the CDF for $z=x_{i+1}$ can be done in $\mathcal{O}(1)$
operations by reusing the previously computed CDF for $z=x_{i}$ and
simply adding the new term $y_{i+1}\varphi_{2,p}(x_{i+1})$ to it.
This technique is known as the \textit{fast sum updating} algorithm
\citep{langrene2019fast}. Its computational complexity is $\mathcal{O}(N\log N)$,
which comes from the sorting of the dataset, and its memory complexity
is $\mathcal{O}(N)$. When the data comes pre-sorted, the computational
complexity of the matrix-vector product is reduced to $\mathcal{O}(N)$.
In this case, the data structure consists only in storing the result of a sort of the data achieved before the optimization of the parameters.
This is a substantial improvement over the $\mathcal{O}(N^{2})$ naive
matrix-vector multiplication algorithm.
\begin{example}
The univariate Mat\'ern-1/2 kernel (a.k.a. Laplacian kernel, or double-exponential
kernel) $K_{1/2}(u)=k_{1/2}(\left|u\right|)=e^{-\left|u\right|}$,
$u\in\mathbb{R}$, satisfies Assumption~\ref{assu:kernel_decomposition_univariate}:
\begin{equation}
k_{1/2}(u-v)=e^{-(u-v)}=e^{-u}e^{v}=\sum_{p=1}^{P}\varphi_{1,p}(u)\varphi_{2,p}(v)\label{eq:k12_decomposition}
\end{equation}
with $P=1$ and, for all $u\in\mathbb{R}$ and $v\in\mathbb{R}$,
\begin{equation}
\begin{array}{lll}
\varphi_{1,1}(u)=e^{-u} & , & \varphi_{2,1}(v)=e^{v}\end{array}\ .\label{eq:k12_decomposition_phi}
\end{equation}
\end{example}
\begin{example}
The univariate Mat\'ern-3/2 kernel $K_{3/2}(u)=k_{3/2}(\left|u\right|)=$\\
$\left(1+\sqrt{3}\left|u\right|\right)e^{-\sqrt{3}\left|u\right|}$,
$u\in\mathbb{R}$, satisfies Assumption~\ref{assu:kernel_decomposition_univariate}:
\begin{align}
k_{3/2}(u-v) & =\left(1+\sqrt{3}(u-v)\right)e^{-\sqrt{3}(u-v)}\nonumber \\
 & =\left(1+\sqrt{3}u\right)e^{-\sqrt{3}u}e^{\sqrt{3}v}-\sqrt{3}e^{-\sqrt{3}u}ve^{\sqrt{3}v}=\sum_{p=1}^{P}\varphi_{1,p}(u)\varphi_{2,p}(v)\label{eq:k32_decomposition}
\end{align}
with $P=2$ and, for all $u\in\mathbb{R}$ and $v\in\mathbb{R}$,
\begin{equation}
\begin{array}{lll}
\varphi_{1,1}(u)=(1+\sqrt{3}u)e^{-\sqrt{3}u} & , & \varphi_{2,1}(v)=e^{\sqrt{3}v}\\
\varphi_{1,2}(u)=-\sqrt{3}e^{-\sqrt{3}u} & , & \varphi_{2,2}(v)=ve^{\sqrt{3}v}\ .
\end{array}\label{eq:k32_decomposition_phi}
\end{equation}
\end{example}
\begin{example}
The univariate Mat\'ern-5/2 kernel $K_{5/2}(u)=k_{5/2}(\left|u\right|)=$\\
$\left(1+\sqrt{5}\left|u\right|+\frac{5}{3}\left|u\right|^{2}\right)e^{-\sqrt{5}\left|u\right|}$,
$u\in\mathbb{R}$, satisfies Assumption~\ref{assu:kernel_decomposition_univariate}:
\begin{align}
 & k_{5/2}(u-v)=\left(1+\sqrt{5}(u-v)+\frac{5}{3}(u-v)^{2}\right)e^{-\sqrt{5}(u-v)}\nonumber \\
 & =\left(1+\sqrt{5}u+\frac{5}{3}u^{2}\right)e^{-\sqrt{5}u}e^{\sqrt{5}v}-\left(\sqrt{5}+\frac{10}{3}u\right)e^{-\sqrt{5}u}ve^{\sqrt{5}v}+\frac{5}{3}e^{-\sqrt{5}u}v^{2}ve^{\sqrt{5}v}\nonumber \\
 & =\sum_{p=1}^{P}\varphi_{1,p}(u)\varphi_{2,p}(v)\label{eq:k52_decomposition}
\end{align}
with $P=3$ and, for all $u\in\mathbb{R}$ and $v\in\mathbb{R}$,
\begin{equation}
\begin{array}{lll}
\varphi_{1,1}(u)=\left(1+\sqrt{5}u+\frac{5}{3}u^{2}\right)e^{-\sqrt{5}u} & , & \varphi_{2,1}(v)=e^{\sqrt{5}v}\\
\varphi_{1,2}(u)=-\left(\sqrt{5}+\frac{10}{3}u\right)e^{-\sqrt{5}u} & , & \varphi_{2,2}(v)=ve^{\sqrt{5}v}\\
\varphi_{1,3}(u)=\frac{5}{3}e^{-\sqrt{5}u} & , & \varphi_{2,3}(v)=v^{2}e^{\sqrt{5}v}\ .
\end{array}\label{eq:k52_decomposition_phi}
\end{equation}
\end{example}
In a similar manner, all univariate Mat\'ern-$\nu$ kernels with
parameter $\nu=n+1/2$, $n\in\mathbb{N}$, satisfy Assumption~\ref{assu:kernel_decomposition_univariate}.
More generally, a list of compatible kernels is provided in \citep{langrene2021fast}.

\subsection{The multivariate case\label{subsec:multivariate-case}}

In this subsection, we explain how the univariate case, addressed above, can be extended to a multivariate setting. We detail the definition of multivariate kernels that we consider and then, after introducing some notation, we detail the computations of the CDF decompositions that we built for the first three analytical Mat\'ern kernels ($\nu\in\{0.5,1.5,2.5\}$).

\subsubsection{Multivariate kernel functions}

In order to extend the definition of kernel functions to the multivariate case, we could consider 
the classical approach \citep{guttorp2006studies}
of applying the univariate kernel to the
Euclidean norm $\left\Vert \mathbf{u}\right\Vert _{2}=\sqrt{\sum_{k=1}^{d}u_{k}^{2}}$
of the vector $\mathbf{u}$:
\begin{equation}
 K(\mathbf{u}):= k(\left\Vert \mathbf{u}\right\Vert _{2}), \label{eq:L2_kernel}
\end{equation}
where $\mathbf{u}=(u_{1},u_{2},\ldots,u_{d})\in\mathbb{R}^{d}$. 
Unfortunately, such a definition is incompatible with our proposed fast matrix-vector multiplication. Two
alternative definitions are the product of univariate kernels 
\begin{equation}
K(\mathbf{u}):=\prod_{k=1}^{d} k( \left |u_{k} \right |), \label{eq:product_kernel}
\end{equation}
as well as the substitution, in the classical multivariate kernel
definition \eqref{eq:L2_kernel}, of the Euclidean norm $\left\Vert \mathbf{u}\right\Vert _{2}$
by the $L1$ norm $\left\Vert \mathbf{u}\right\Vert _{1}=\sum_{k=1}^{d}\left|u_{k}\right|$:
\begin{equation}
K(\mathbf{u}):= k(\left\Vert \mathbf{u}\right\Vert _{1}).\label{eq:L1_kernel}
\end{equation}
Appendix~\ref{sec:positive_definite} shows that if the L2 kernel \eqref{eq:L2_kernel} is positive definite for all $d\geq 1$, then the product kernel \eqref{eq:product_kernel} is also positive definite for all $d\geq 1$. A sufficient condition for L1 kernels of the type \eqref{eq:L1_kernel} to be positive definite is also provided.

Subsubsection~\ref{subsubsec:cdf_decomposition} will describe how definitions \eqref{eq:product_kernel} and \eqref{eq:L1_kernel} are compatible with fast exact kernel decomposition under a suitable assumption about the kernel function $k$. Before doing so, we first need to introduce notations regarding multivariate empirical cumulative distribution functions.

\subsubsection{Multivariate CDF formulas\label{subsubsec:cdf_definition}}

Let $(\boldsymbol{x}_{1},{y}_{1}),(\boldsymbol{x}_{2},{y}_{2}),\ldots,(\boldsymbol{x}_{N},{y}_{N})$
be a sample of $N$ input (source) points $\boldsymbol{x}_{i}=(x_{1,i},{\allowbreak}x_{2,i},\allowbreak\ldots,{\allowbreak}x_{d,i})\in\mathbb{R}^{d}$
and output (response) points ${y}_{i}\in\mathbb{R}$. Consider
an evaluation (target) point $\boldsymbol{z}=(z_{1},z_{2},\ldots,z_{d})\in\mathbb{R}^{d}$.
We define a weighted multivariate empirical cumulative distribution
function (ECDF) as follows:
\begin{equation}
F(\boldsymbol{z})=F(\boldsymbol{z};\mathbf{x},\mathbf{y}):=\sum_{i=1}^{N}y_{i}\mathbbm{1}\{x_{1,i}\leq z_{1},\ldots,x_{d,i}\leq z_{d}\}\,.\label{eq:ECDF}
\end{equation}
The particular case $\mathbf{y}\equiv1/N$ corresponds to the classical
joint empirical distribution function $F(\boldsymbol{z})=\frac{1}{N}\sum_{i=1}^{N}\mathbbm{1}\{x_{1,i}\leq z_{1},\ldots,x_{d,i}\leq z_{d}\}$. More generally, define the following multivariate ECDF:
\begin{align}
F(\boldsymbol{z},\boldsymbol{\delta}) & =F(\mathbf{z},\boldsymbol{\delta};\mathbf{x},\mathbf{y}):=\sum_{i=1}^{N}y_{i}\mathbbm{1}\{x_{1,i}\leq_{\delta_{1}}z_{1},\ldots,x_{d,i}\leq_{\delta_{d}}z_{d}\}\label{eq:ECDFdelta}
\end{align}
where $\boldsymbol{\delta}=\left\{ \delta_{1},\delta_{2},\ldots,\delta_{d}\right\} \in\left\{ -1,1\right\} ^{d}$,
and where the generalized inequality operator $\leq_{c}$ corresponds
to $\leq$ (lower or equal) if $c\geq0$, and to $<$ (strictly lower)
if $c<0$. In particular $F(\boldsymbol{z})=F(\boldsymbol{z},\mathbf{1})$. We then introduce the dot products and element-wise products
\begin{eqnarray*}
\boldsymbol{\delta.z}=\sum_{k=1}^{d}\delta_{k}z_{k}\in\mathbb{R} & , & \boldsymbol{\delta z}=(\delta_{1}z_{1},\delta_{2}z_{2},\ldots,\delta_{d}z_{d})\in\mathbb{R}^{d}\\
\boldsymbol{\delta.x}_{i}=\sum_{k=1}^{d}\delta_{k}x_{k,i}\in\mathbb{R} & , & \boldsymbol{\delta x}_{i}=(\delta_{1}x_{1,i},\delta_{2}x_{2,i},\ldots,\delta_{d}x_{d,i})\in\mathbb{R}^{d}
\end{eqnarray*}

\subsubsection{Fast CDF decomposition\label{subsubsec:cdf_decomposition}}

Let $K$ be a multivariate kernel satisfying the following assumption.
\begin{assumption}
\label{assu:kernel_decomposition_multivariate}The shift-invariant multivariate kernel $K$ is defined either
as a product kernel $K(\mathbf{u})=\prod_{k=1}^{d} k( \left |u_{k} \right |)$ or as an L1 kernel $K(\mathbf{u})=k(\left\Vert \mathbf{u}\right\Vert _{1})$, where $k:\mathbb{R}\rightarrow\mathbb{R}$ satisfies equation~\eqref{eq:kernel_decomposition_assumption}.
\end{assumption}
\begin{prop}
\label{prop:L1_kernel_decomposition}Under Assumption~\ref{assu:kernel_decomposition_multivariate},
the kernel matrix-vector product \eqref{eq:Ky}-\eqref{eq:kde} can
be decomposed, in the L1 kernel case, as
\begin{align}
 & \sum_{i=1}^{N}y_{i}K(\boldsymbol{x}_{i}-\boldsymbol{z})=\sum_{i=1}^{N}y_{i}k\left(\left\Vert \boldsymbol{x}_{i}-\boldsymbol{z}\right\Vert _{1}\right)=\sum_{i=1}^{N}y_{i}k\!\left(\sum_{k=1}^{d}\left|x_{k,i}-z_{k}\right|\right)\nonumber \\
= & \sum_{\boldsymbol{\delta}\in\{-1,1\}^{d}}\sum_{p=1}^{P}\varphi_{1,p}(\boldsymbol{\delta.z})F(\boldsymbol{\delta}\boldsymbol{z},\boldsymbol{\delta};\boldsymbol{\delta}\mathbf{x},\boldsymbol{w}_p)\label{eq:L1_decomposition}
\end{align}
where the weights vector $\boldsymbol{w}_p=(w_{1,p},w_{2,p},\ldots,w_{d,p})\in\mathbb{R}^{d}$
is defined by
\begin{equation}
w_{i,p}=w_{i,p}(\boldsymbol{\delta},\boldsymbol{x}_{i},y_{i}):=y_{i}\varphi_{2,p}(\boldsymbol{\delta.x}_{i}).\label{eq:L1_weights}
\end{equation}
\end{prop}
\begin{proof}
\begin{align*}
 & \sum_{i=1}^{N}y_{i}K(\boldsymbol{x}_{i}-\boldsymbol{z})=\sum_{i=1}^{N}y_{i}k\left(\left\Vert \boldsymbol{x}_{i}-\boldsymbol{z}\right\Vert _{1}\right)=\sum_{i=1}^{N}y_{i}k\!\left(\sum_{k=1}^{d}\left|x_{k,i}-z_{k}\right|\right)\\
= & \sum_{\boldsymbol{\delta}\in\{-1,1\}^{d}}\sum_{i=1}^{N}y_{i}k\!\left(\sum_{k=1}^{d}\left(\delta_{k}z_{k}-\delta_{k}x_{k,i}\right)\right)\mathbbm{1}\{\delta_{1}x_{1,i}\leq_{\delta_{1}}\delta_{1}z_{1},\ldots,\delta_{d}x_{d,i}\leq_{\delta_{d}}\delta_{d}z_{d}\}\\
= & \sum_{\boldsymbol{\delta}\in\{-1,1\}^{d}}\sum_{i=1}^{N}y_{i}\sum_{p=1}^{P}\varphi_{1,p}(\boldsymbol{\delta.z})\varphi_{2,p}(\boldsymbol{\delta.x}_{i})\mathbbm{1}\{\delta_{1}x_{1,i}\leq_{\delta_{1}}\delta_{1}z_{1},\ldots,\delta_{d}x_{d,i}\leq_{\delta_{d}}\delta_{d}z_{d}\}\\
= & \sum_{\boldsymbol{\delta}\in\{-1,1\}^{d}}\sum_{p=1}^{P}\varphi_{1,p}(\boldsymbol{\delta.z})\sum_{i=1}^{N}y_{i}\varphi_{2,p}(\boldsymbol{\delta.x}_{i})\mathbbm{1}\{\delta_{1}x_{1,i}\leq_{\delta_{1}}\delta_{1}z_{1},\ldots,\delta_{d}x_{d,i}\leq_{\delta_{d}}\delta_{d}z_{d}\}\\
= & \sum_{\boldsymbol{\delta}\in\{-1,1\}^{d}}\sum_{p=1}^{P}\varphi_{1,p}(\boldsymbol{\delta.z})F(\boldsymbol{\delta}\boldsymbol{z},\boldsymbol{\delta};\boldsymbol{\delta}\mathbf{x},\boldsymbol{w}_p)
\end{align*}
\end{proof}
\begin{prop}
\label{prop:product_kernel_decomposition}Under Assumption~\ref{assu:kernel_decomposition_multivariate},
the kernel matrix-vector product \eqref{eq:Ky}-\eqref{eq:kde} can
be decomposed, in the product kernel case, as
\begin{align}
 & \sum_{i=1}^{N}y_{i}K(\boldsymbol{x}_{i}-\boldsymbol{z})=\sum_{i=1}^{N}y_{i}\prod_{k=1}^{d}k(\left|x_{k,i}-z_{k}\right|)\nonumber \\
= & \sum_{\boldsymbol{\theta}\in\{1,2,\ldots,P\}^{d}}\sum_{\boldsymbol{\delta}\in\{-1,1\}^{d}}\left(\prod_{k=1}^{d}\varphi_{1,\theta_{k}}(\delta_{k}z_{k})\right)F(\boldsymbol{\delta}\boldsymbol{z},\boldsymbol{\delta};\boldsymbol{\delta}\mathbf{x},\boldsymbol{w})\label{eq:product_decomposition}
\end{align}
where the weights vector $\boldsymbol{w}=(w_{1},w_{2},\ldots,w_{d})\in\mathbb{R}^{d}$
is defined by
\begin{equation}
w_{i}=w_{i}(\boldsymbol{\delta},\boldsymbol{\theta},\boldsymbol{x}_{i},y_{i}):=y_{i}\prod_{k=1}^{d}\varphi_{2,\theta_{k}}(\delta_{k}x_{k,i}).\label{eq:product_weights}
\end{equation}
\end{prop}
\begin{proof}
\begin{align*}
 & \sum_{i=1}^{N}y_{i}K(\boldsymbol{x}_{i}-\boldsymbol{z})=\sum_{i=1}^{N}y_{i}\prod_{k=1}^{d}k(\left|x_{k,i}-z_{k}\right|)\\
= & \sum_{\boldsymbol{\delta}\in\{-1,1\}^{d}}\sum_{i=1}^{N}y_{i}\prod_{k=1}^{d}k(\delta_{k}z_{k}-\delta_{k}x_{k,i})\mathbbm{1}\{\delta_{1}x_{1,i}\leq_{\delta_{1}}\delta_{1}z_{1},\ldots,\delta_{d}x_{d,i}\leq_{\delta_{d}}\delta_{d}z_{d}\}\\
= & \sum_{\boldsymbol{\delta}\in\{-1,1\}^{d}}\sum_{i=1}^{N}y_{i}\prod_{k=1}^{d}\left(\sum_{p=1}^{P}\varphi_{1,p}(\delta_{k}z_{k})\varphi_{2,p}(\delta_{k}x_{k,i})\right)\mathbbm{1}\{\delta_{1}x_{1,i}\leq_{\delta_{1}}\delta_{1}z_{1},\ldots,\delta_{d}x_{d,i}\leq_{\delta_{d}}\delta_{d}z_{d}\}\\
= & \sum_{\boldsymbol{\delta}\in\{-1,1\}^{d}}\sum_{i=1}^{N}y_{i}\sum_{\boldsymbol{\theta}\in\{1,2,\ldots,P\}^{d}}\left(\prod_{k=1}^{d}\varphi_{1,\theta_{k}}(\delta_{k}z_{k})\varphi_{2,\theta_{k}}(\delta_{k}x_{k,i})\right)\mathbbm{1}\{\delta_{1}x_{1,i}\leq_{\delta_{1}}\delta_{1}z_{1},\ldots,\delta_{d}x_{d,i}\leq_{\delta_{d}}\delta_{d}z_{d}\}\\
= & \sum_{\boldsymbol{\theta}\in\{1,2,\ldots,P\}^{d}}\sum_{\boldsymbol{\delta}\in\{-1,1\}^{d}}\left(\prod_{k=1}^{d}\varphi_{1,\theta_{k}}(\delta_{k}z_{k})\right)F(\boldsymbol{\delta}\boldsymbol{z},\boldsymbol{\delta};\boldsymbol{\delta}\mathbf{x},\boldsymbol{w})
\end{align*}
\end{proof}
Propositions~\ref{prop:L1_kernel_decomposition} and \ref{prop:product_kernel_decomposition} show that, under Assumption~\ref{assu:kernel_decomposition_multivariate}, the kernel matrix-vector product \eqref{eq:kde} can be written as a weighted sum of empirical cumulative distribution functions. This generalizes equation~\eqref{eq:fast_exact_decomposition_1d} to the multidimensional setting. Figure~\ref{fig:kernel_decomposition}
provides an intuitive description of the kernel CDF decomposition \eqref{eq:L1_decomposition} for bivariate
data.

\begin{figure}[H]
\begin{minipage}[t]{0.48\columnwidth}%
\includegraphics[width=0.38\paperwidth]{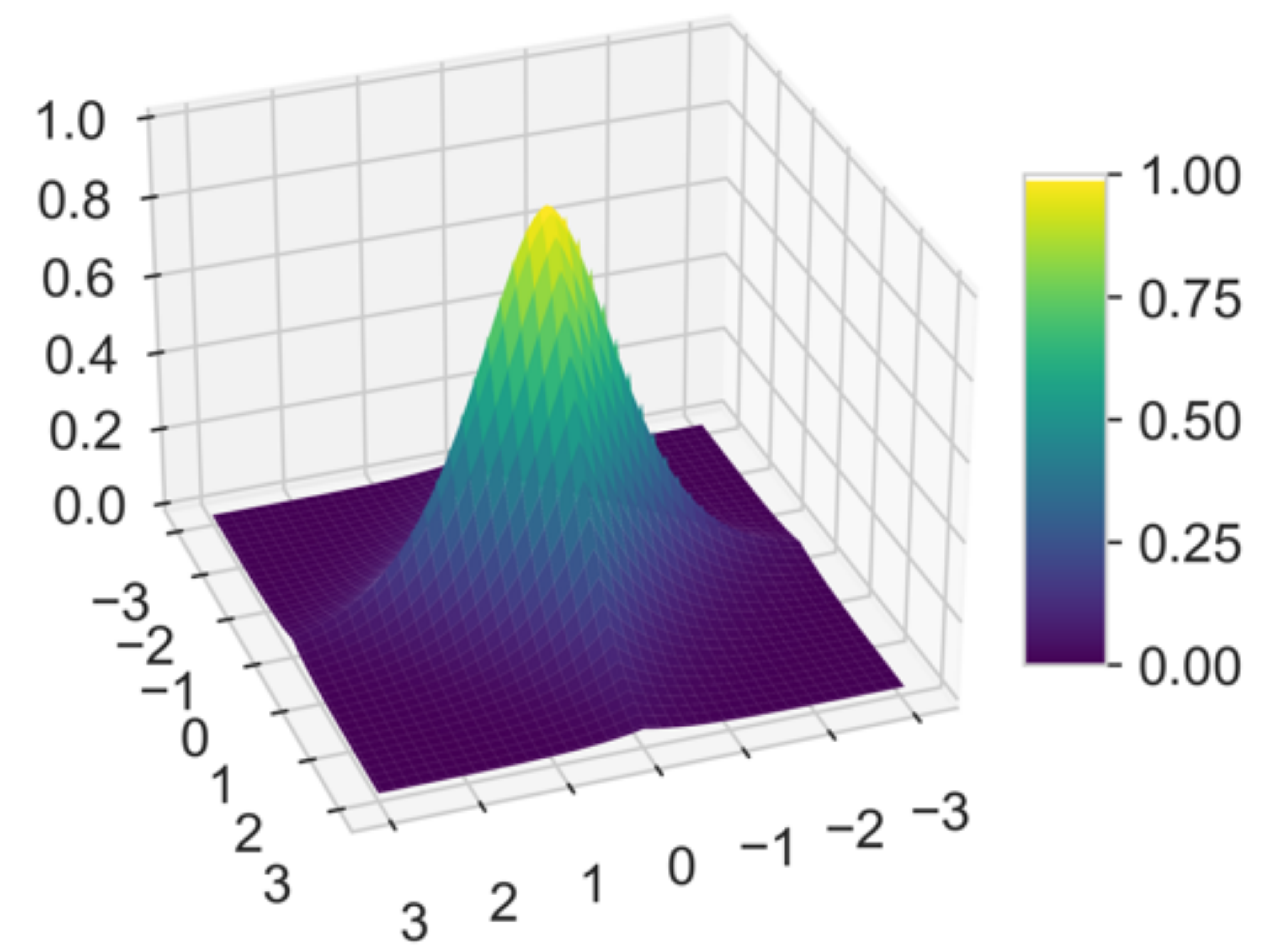}%
\end{minipage}\hfill{}%
\begin{minipage}[t]{0.48\columnwidth}%
\includegraphics[width=0.38\paperwidth]{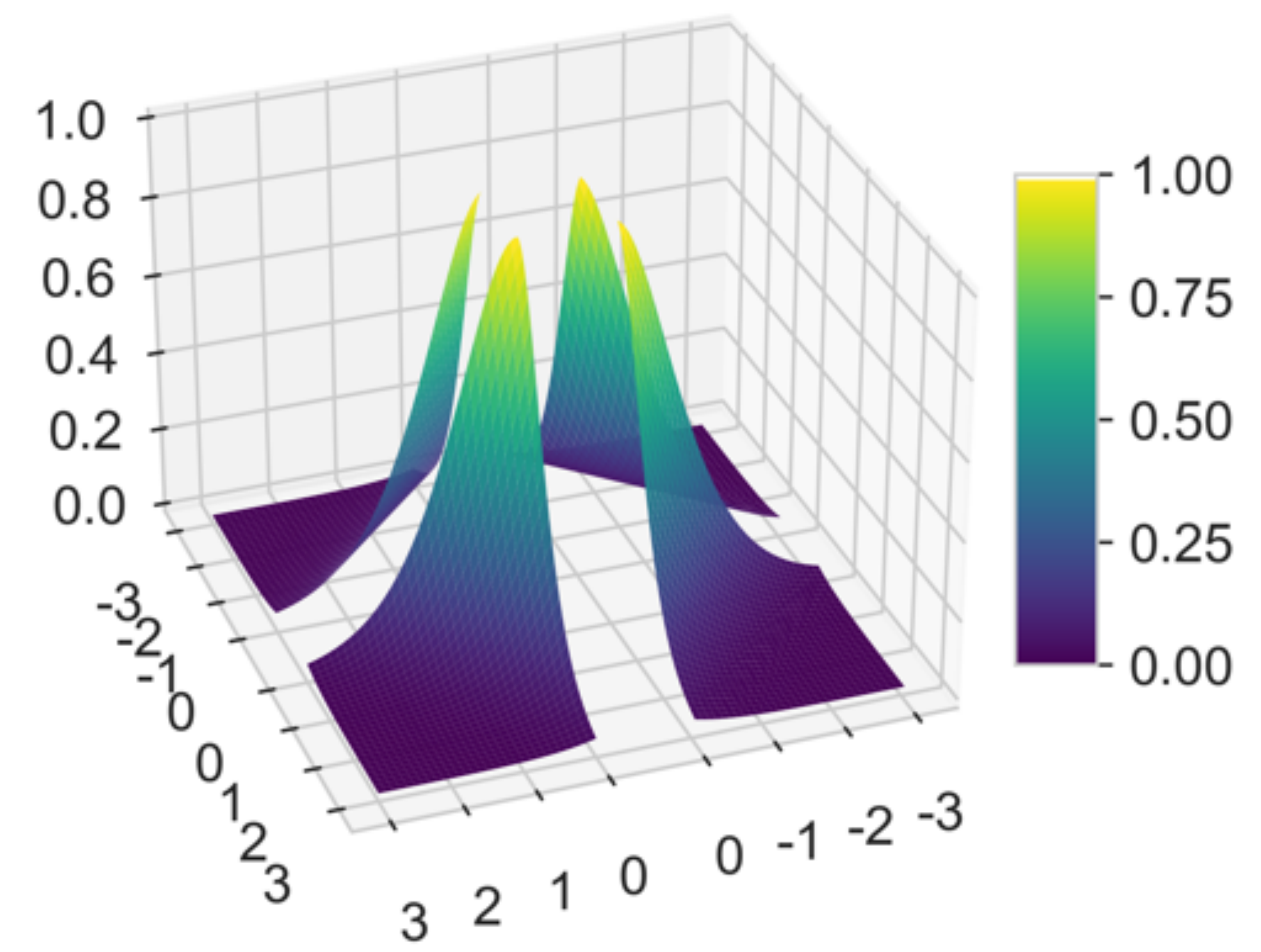}%
\end{minipage}

\caption{Left: bivariate Mat\'ern-5/2 kernel; Right: its decomposition into
four weighted CDFs: $\mathbb{P}(.\protect\leq x_{1},.\protect\leq x_{2})$,
$\mathbb{P}(.>x_{1},.\protect\leq x_{2})$, $\mathbb{P}(.\protect\leq x_{1},.>x_{2})$
and $\mathbb{P}(.>x_{1},.>x_{2})$. Computing the matrix-vector multiplication
\eqref{eq:Ky} costs $\mathcal{O}(N^{2})$ operations with the left-hand
side formulation, but only $\mathcal{O}(N\log N)$ operations with
the equivalent right-hand side formulation.\label{fig:kernel_decomposition}}
\end{figure}

What is now needed is a fast algorithm to compute multivariate CDFs. On multivariate datasets, empirical CDFs cannot anymore be computed
by data sorting and fast sum updating. Instead, a multivariate extension
by a divide-and-conquer algorithm was proposed in \citet{bentley1980divide}, \citet{bouchard2012monte}.
It makes it possible to compute $d$-dimensional empirical CDFs in
$O(N\log(N)^{d-1})$ operations for $d>1$. \citet{langrene2021fast}
provided a detailed implementation for the $d$-dimensional case and
showed how to compute all the required CDFs simultaneously.

In this paper, we further improve the implementation of the algorithm \cite{langrene2021fast} by storing the sorted subsets for each coordinate in a data structure. We then use the fast CDF computation algorithm \cite{langrene2021fast} where all the sorting steps are replaced by these stored results. While the computational complexity remains $O(N\log(N)^{d-1})$, this presorting step greatly reduces the complexity's multiplicative constant (by a factor of tens according to our experiments).

%\red{Similarly to the one-dimensional case, the divide-and-conquer algorithm \cite{bouchard2012monte} is applied to the data set $\mathbf{x}_{1},\mathbf{x}_{2},\ldots,\mathbf{x}_{N}\in\mathbb{R}^{d}$, before the calibration of the model. Almost all of the computation time is spent sorting subsets of the data for each coordinate. The results of all the sorts are stored in arrays within a data structure. During calibration, the divide-and-conquer approach of \cite{bouchard2012monte} is used in the CDF computation algorithm \cite{langrene2021fast}, replacing all the sorting steps with the stored results. In dimensions greater than one, the complexity remains   in $\mathcal{O}(N\times(\log N)^{d-1})$ as it is driven by the divide-and-conquer approach. However, the complexity's multiplicative constant is reduced by a factor of tens.}

%\orange{``Moreover, since all matrix-vector multiplications are using the same data set $\mathbf{x}_{1},\mathbf{x}_{2},\ldots,\mathbf{x}_{N}\in\mathbb{R}^{d}$, a data structure is used to pre-compute all the necessary sorting in all dimensions used in the divide-and-conquer approach. While it does not change the overall complexity of the algorithm for $d>1$, it reduces the multiplicative constant that affects this complexity.''}

To sum
up, after data pre-sorting, matrix-vector multiplications of the type \eqref{eq:Ky}-\eqref{eq:kde}
can be computed in $O(N\log(N)^{d-1})$ operations using the exact fast kernel CDF decomposition formulas \eqref{eq:L1_decomposition}-\eqref{eq:product_decomposition}. Remark that since this fast kernel
MVM algorithm is exact (it does not introduce any approximation),
it sets the minimum computational speed that any alternative fast
approximate MVM algorithm should be capable of achieving in order to be competitive in practice.

One downside of this approach is the increased computational complexity as the
dimension $d$ of the problem grows, as seen from the $O(N\log(N)^{d-1})$
computational complexity, as well as the number of terms in the outer sums in formulas \eqref{eq:L1_decomposition}-\eqref{eq:product_decomposition} ($2^d P$ for the L1 formulation \eqref{eq:L1_kernel}, $(2P)^d$ for the product formulation \eqref{eq:product_kernel}). This makes this approach better suited for
``large $N$, small $d$'' types of problems, such as those encountered in spatial statistics and similar fields.

Another downside  is that not all stationary covariance
functions are compatible, since Assumption~\ref{assu:kernel_decomposition_multivariate} needs to be satisfied. Among infinite-support kernels, the only compatible kernels, to our knowledge, are defined as the product of the Laplacian kernel with another compatible kernel (such as a polynomial). This rules out the Gaussian kernel (a.k.a. squared exponential covariance
function), but includes all Mat\'ern kernels with $\nu=p+1/2$, $p\in\mathbb{N}$. Among compact kernels, we know from \cite{langrene2019fast} that a large class of kernels are compatible, including the class of symmetric beta kernels (uniform, triangular, parabolic, etc.).

\begin{comment}
Gaussian process regression requires that $K$ in \eqref{eq:gp_f} is positive definite to guarantee that the matrix $\mathbf{K}+\sigma^{2}\mathbf{I}$ is invertible for all $\sigma$ values. 
Defining multidimensional kernels with \eqref{eq:product_kernel} ensure that the kernel are positive definite too as the product of kernel is a kernel.  Between these two, the $L1$ formulation
\eqref{eq:L1_kernel} is the most computationally attractive as it
generates much fewer CDF terms to compute than the product formulation
\eqref{eq:product_kernel}. 
\end{comment}

%In the next section we provide the detailed formulas of the fast matrix-vector multiplication for the  Mat\'ern covariance kernels with $\nu\in\{1/2,3/2,5/2\}$, defined using the $L1$ approach \eqref{eq:L1_kernel}, and their gradients.

\subsection{Multivariate  Mat\'ern kernel decomposition formulas\label{subsec:multivariate_matern_decomposition}}

Mat\'ern kernels are an important particular case of kernel functions satisfying Assumption~\ref{assu:kernel_decomposition_multivariate}. For convenience, we provide in this subsection the detail of the decomposition of several examples of multivariate Mat\'ern kernel covariance functions into sums of weighted empirical cumulative distribution functions, as established in equations \eqref{eq:L1_decomposition} and \eqref{eq:product_decomposition}. We also provide a similar decomposition for their gradient with respect to the lengthscale parameter $\ell$, see equation \eqref{eq:scaled_kernel}.

\begin{example} % Matérn-1/2 kernel
The scaled multivariate Mat\'ern-1/2 covariance kernel is given by
\begin{equation}
K_{1/2;\varsigma,\ell}(\mathbf{u})=\varsigma^{2}K_{1/2}(\mathbf{u}/\ell)=\varsigma^{2}k_{1/2}(\left\Vert \mathbf{u}\right\Vert _{1}/\ell)=\varsigma^{2}\exp\left(-\frac{1}{\ell}\sum_{k=1}^{d}\left|u_{k}\right|\right), \ \mathbf{u}\in\mathbb{R}^{d}\label{eq:matern_l1_12}
\end{equation}
Remark that this is the only multivariate kernel for which the product formulation \eqref{eq:product_kernel} and the L1 formulation \eqref{eq:L1_kernel} coincide. 
Using equations~\eqref{eq:k12_decomposition}-\eqref{eq:k12_decomposition_phi} and Proposition~\ref{prop:L1_kernel_decomposition}, the kernel sum $\sum_{i=1}^{N}y_{i}K_{1/2;\varsigma,\ell}(\mathbf{x}_{i}-\mathbf{z})$ admits the decomposition
\begin{align}
 & \sum_{i=1}^{N}y_{i}K_{1/2;\varsigma,\ell}(\mathbf{x}_{i}-\mathbf{z}) =  \varsigma^{2}\sum_{i=1}^{N}y_{i}\exp\left(-\frac{1}{\ell}\sum_{k=1}^{d}\left|x_{k,i}-z_{k}\right|\right)\nonumber \\
& = \ \varsigma^{2}\!\!\!\!\sum_{\boldsymbol{\delta}\in\{-1,1\}^{d}}\!\!\!\!\exp\left(-\frac{\boldsymbol{\delta.z}}{\ell}\right)F(\boldsymbol{\delta}\boldsymbol{z},\boldsymbol{\delta};\boldsymbol{\delta}\mathbf{x},\boldsymbol{w}^{(0)})\label{eq:matern_12_decomposition}
\end{align}
where the weights vector $\boldsymbol{w}^{(p)}=(w_{1}^{(p)},w_{2}^{(p)},\ldots,w_{d}^{(p)})\in\mathbb{R}^{d}$ is defined by
\begin{equation}
w_{i}^{(p)}=w_{i}^{(p)}(\nu,\boldsymbol{\delta}):=y_{i}\left(\boldsymbol{\delta.x}_{i}\right)^{p}\exp\left(\frac{\sqrt{2\nu}}{\ell}\boldsymbol{\delta.x}_{i}\right)\label{eq:matern_weights}
\end{equation}
for any $p\in\mathbb{N}$ and $i=1,2,\ldots,N$, where $\nu$ is the parameter of the Mat\'ern-$\nu$ kernel (in this example, $\nu=1/2$).
Similarly, the gradient with respect to the lengthscale $\ell$ can be shown to be equal to
\begin{align*}
 & \frac{\partial}{\partial\ell}\left(\sum_{i=1}^{N}y_{i}K_{1/2;\varsigma,\ell}(\mathbf{x}_{i}-\mathbf{z})\right)=\sum_{i=1}^{N}y_{i}\frac{\varsigma^{2}}{\ell^{2}}\left\Vert \mathbf{x}_{i}-\mathbf{z}\right\Vert _{1}\exp\left(-\frac{1}{\ell}\left\Vert \mathbf{x}_{i}-\mathbf{z}\right\Vert _{1}\right)\\
 & =\frac{\varsigma^{2}}{\ell^{2}}\sum_{\boldsymbol{\delta}\in\{-1,1\}^{d}}\exp\left(-\frac{1}{\ell}\boldsymbol{\delta.z}\right)\left(\boldsymbol{\delta.z}F(\boldsymbol{\delta}\boldsymbol{z},\boldsymbol{\delta};\boldsymbol{\delta}\mathbf{x},\boldsymbol{w}^{(0)})-F(\boldsymbol{\delta}\boldsymbol{z},\boldsymbol{\delta};\boldsymbol{\delta}\mathbf{x},\boldsymbol{w}^{(1)})\right).
\end{align*}
where $\boldsymbol{w}^{(0)}$ and $\boldsymbol{w}^{(1)}$ are defined by equation~\eqref{eq:matern_weights}.
\end{example}

\begin{example} % L1 Matérn-3/2 kernel
The scaled multivariate L1 Mat\'ern-3/2 covariance kernel is given by
\begin{equation}
K_{3/2;\varsigma,\ell}(\mathbf{u})=\varsigma^{2}k_{3/2}(\left\Vert \mathbf{u}\right\Vert _{1}/\ell)=\varsigma^{2}\left(1+\frac{\sqrt{3}}{\ell}\sum_{k=1}^{d}\left|u_{k}\right|\right)\exp\left(-\frac{\sqrt{3}}{\ell}\sum_{k=1}^{d}\left|u_{k}\right|\right), \ \mathbf{u}\in\mathbb{R}^{d}\label{eq:matern_l1_32}
\end{equation}
Using equations~\eqref{eq:k32_decomposition}-\eqref{eq:k32_decomposition_phi} and Proposition~\ref{prop:L1_kernel_decomposition}, the kernel sum $\sum_{i=1}^{N}y_{i}K_{3/2;\varsigma,\ell}(\mathbf{x}_{i}-\mathbf{z})$ admits the decomposition
\begin{align}
 & \sum_{i=1}^{N}y_{i}K_{3/2;\varsigma,\ell}(\mathbf{x}_{i}-\mathbf{z}) = \varsigma^{2}\sum_{i=1}^{N}y_{i}\left(1+\frac{\sqrt{3}}{\ell}\sum_{k=1}^{d}\left|x_{k,i}-z_{k}\right|\right)\exp\left(-\frac{\sqrt{3}}{\ell}\sum_{k=1}^{d}\left|x_{k,i}-z_{k}\right|\right)\nonumber \\
 & = \  \varsigma^{2}\!\!\!\!\sum_{\boldsymbol{\delta}\in\{-1,1\}^{d}}\!\!\!\!\exp\!\left(-\frac{\sqrt{3}}{\ell}\boldsymbol{\delta.z}\right)\left(\left(1+\frac{\sqrt{3}}{\ell}\boldsymbol{\delta.z}\right)F(\boldsymbol{\delta}\boldsymbol{z},\boldsymbol{\delta};\boldsymbol{\delta}\mathbf{x},\boldsymbol{w}^{(0)})-\frac{\sqrt{3}}{\ell}F(\boldsymbol{\delta}\boldsymbol{z},\boldsymbol{\delta};\boldsymbol{\delta}\mathbf{x},\boldsymbol{w}^{(1)})\right)\label{eq:matern_32_decomposition}
\end{align}
where $\boldsymbol{w}^{(0)}$ and $\boldsymbol{w}^{(1)}$ are defined by equation~\eqref{eq:matern_weights}.
Similarly, the gradient with respect to the lengthscale $\ell$ can be shown to be equal to
\begin{align*}
 & \frac{\partial}{\partial\ell}\left(\sum_{i=1}^{N}y_{i}K_{3/2;\varsigma,\ell}(\mathbf{x}_{i}-\mathbf{z})\right)=\sum_{i=1}^{N}3y_{i}\frac{\varsigma^{2}}{\ell^{3}}\left\Vert \mathbf{x}_{i}-\mathbf{z}\right\Vert _{1}^{2}\exp\left(-\frac{\sqrt{3}}{\ell}\left\Vert \mathbf{x}_{i}-\mathbf{z}\right\Vert _{1}\right)\\
 & =3\frac{\varsigma^{2}}{\ell^{3}}\sum_{\boldsymbol{\delta}\in\{-1,1\}^{d}}\exp\left(-\frac{\sqrt{3}}{\ell}\boldsymbol{\delta.z}\right)\left[\left(\boldsymbol{\delta.z}\right)^{2}F(\boldsymbol{\delta}\boldsymbol{z},\boldsymbol{\delta};\boldsymbol{\delta}\mathbf{x},\boldsymbol{w}^{(0)})\right.\\
 & \left.-2\boldsymbol{\delta.z}F(\boldsymbol{\delta}\boldsymbol{z},\boldsymbol{\delta};\boldsymbol{\delta}\mathbf{x},\boldsymbol{w}^{(1)})+F(\boldsymbol{\delta}\boldsymbol{z},\boldsymbol{\delta};\boldsymbol{\delta}\mathbf{x},\boldsymbol{w}^{(2)})\right].
\end{align*}
\end{example}
\vspace{2mm}

\begin{example} % Product Matérn-3/2 kernel
The scaled multivariate product L1 Mat\'ern-3/2 covariance kernel is given by
\begin{equation}
K^{\Pi}_{3/2;\varsigma,\ell}(\mathbf{u})=\varsigma^{2}\prod_{k=1}^{d}k_{3/2}(\left|u_{k}\right|/\ell)=\varsigma^{2}\prod_{k=1}^{d}\left(1+\frac{\sqrt{3}}{\ell}\left|u_{k}\right|\right)\exp\!\left(-\frac{\sqrt{3}}{\ell}\sum_{k=1}^{d}\left|u_{k}\right|\right), \ \mathbf{u}\in\mathbb{R}^{d}\label{eq:eq:matern_product_32}
\end{equation}
For any $\mathbf{z}\in\mathbb{R}^{d}$, $\mathbf{x}_{i}\in\mathbb{R}^{d}$,
$\boldsymbol{\theta}\in\{0,1\}^{d}$ and $\boldsymbol{\delta}\in\{-1,1\}^{d}$,
introduce the weights $\pi^{(0)}(\mathbf{z})=\pi^{(0)}(\mathbf{z};\nu,\boldsymbol{\theta},\boldsymbol{\delta})\in\mathbb{R}$
and $\pi^{(1)}(\boldsymbol{x}_{i})=\pi^{(1)}(\boldsymbol{x}_{i};\nu,\boldsymbol{\theta},\boldsymbol{\delta})\in\mathbb{R}$
defined by
\begin{align}
\pi^{(0)}(\mathbf{z}) & =\prod_{k=1}^{d}\theta_{k}\left(1+\frac{\sqrt{2\nu}}{\ell}\delta_{k}z_{k}\right),\label{eq:matern_product_weights_0}\\
\pi^{(1)}(\boldsymbol{x}_{i}) & =\prod_{k=1}^{d}(1-\theta_{k})\left(-\frac{\sqrt{2\nu}}{\ell}\delta_{k}x_{k,i}\right),\label{eq:matern_product_weights_1}
\end{align}
and the weights vector $\boldsymbol{w}=(w_{1},w_{2},\ldots,w_{d})\in\mathbb{R}^{d}$
defined by
\begin{equation}
w_{i}=w_{i}(\nu,\boldsymbol{\theta},\boldsymbol{\delta}):=y_{i}\pi^{(1)}(\boldsymbol{x}_{i};\nu,\boldsymbol{\theta},\boldsymbol{\delta})\exp\left(\frac{\sqrt{2\nu}}{\ell}\boldsymbol{\delta.x}_{i}\right).\label{eq:matern_product_weights_vector}
\end{equation}
where $\nu=3/2$ in this example. Using these definitions, as well as equations~\eqref{eq:k32_decomposition}-\eqref{eq:k32_decomposition_phi} and Proposition~\ref{prop:product_kernel_decomposition}, the kernel sum $\sum_{i=1}^{N}y_{i}K^{\Pi}_{3/2;\varsigma,\ell}(\mathbf{x}_{i}-\mathbf{z})$ can be shown to admit the decomposition
\begin{align}
 & \sum_{i=1}^{N}y_{i}K_{3/2;\varsigma,\ell}^{\Pi}(\mathbf{x}_{i}-\mathbf{z})
=  \varsigma^{2}\sum_{i=1}^{N}y_{i}\prod_{k=1}^{d}\left(1+\frac{\sqrt{3}}{\ell}\left|x_{k,i}-z_{k}\right|\right)\exp\left(-\frac{\sqrt{3}}{\ell}\sum_{k=1}^{d}\left|x_{k,i}-z_{k}\right|\right)\nonumber \\
 & =\varsigma^{2}\!\!\sum_{\boldsymbol{\theta}\in\{0,1\}^{d}}\sum_{\boldsymbol{\delta}\in\{-1,1\}^{d}}\pi^{(0)}(\mathbf{z};\frac{3}{2},\boldsymbol{\theta},\boldsymbol{\delta})\exp\!\left(-\frac{\sqrt{3}}{\ell}\boldsymbol{\delta.z}\right)F(\boldsymbol{\delta}\boldsymbol{z},\boldsymbol{\delta};\boldsymbol{\delta}\mathbf{x},\boldsymbol{w})\label{eq:matern_product_32_decomposition}
\end{align}
Similarly, the gradient with respect to the lengthscale $\ell$ can be shown to be equal to
\begin{align*}
 & \frac{\partial}{\partial\ell}\left(\sum_{i=1}^{N}y_{i}K^{\Pi}_{3/2;\varsigma,\ell}(\mathbf{x}_{i}-\mathbf{z})\right)\\
 & =\frac{\varsigma^{2}}{\ell}\sum_{\boldsymbol{\theta}\in\{0,1\}^{d}}\sum_{\boldsymbol{\delta}\in\{-1,1\}^{d}}\sum_{k=1}^{d}\theta_{k}\frac{\left(\frac{\sqrt{3}}{\ell}\delta_{k}z_{k}\right)^{2}}{\left(1+\frac{\sqrt{3}}{\ell}\delta_{k}z_{k}\right)}\pi^{(0)}(\mathbf{z})\exp\left(-\frac{\sqrt{3}}{\ell}\boldsymbol{\delta.z}\right)F(\boldsymbol{\delta}\boldsymbol{z},\boldsymbol{\delta};\boldsymbol{\delta}\mathbf{x},\acute{\boldsymbol{w}}_{k})
\end{align*}
where the weights vector $\acute{\boldsymbol{w}}_{k}=(\acute{w}_{k,1},\acute{w}_{k,2},\ldots,\acute{w}_{k,d})\in\mathbb{R}^{d}$
is defined by 
\begin{equation}
\acute{w}_{k,i}=\acute{w}_{k,i}(\nu,\boldsymbol{\theta},\boldsymbol{\delta}):=y_{i}\left(-\frac{\sqrt{3}}{\ell}(1-\theta_{k})\delta_{k}x_{k,i}\right)\pi^{(1)}(\boldsymbol{x}_{i};\nu,\boldsymbol{\theta},\boldsymbol{\delta})\exp\!\left(\frac{\sqrt{2\nu}}{\ell}\boldsymbol{\delta.x}_{i}\right).\label{eq:matern_product_weights_vector_derivative}
\end{equation}
\end{example}

\vspace{5mm}
\begin{example} % L1 Matérn-5/2 kernel
The scaled multivariate L1 Mat\'ern-5/2 covariance kernel is given by
\begin{align}
 & K_{5/2;\varsigma,\ell}(\mathbf{u})=\varsigma^{2}K_{5/2}(\mathbf{u}/\ell)=\varsigma^{2}k_{5/2}(\left\Vert \mathbf{u}\right\Vert _{1}/\ell)\nonumber \\
 & =\varsigma^{2}\left(1+\frac{\sqrt{5}}{\ell}\sum_{k=1}^{d}\left|u_{k}\right|+\frac{5}{3\ell^{2}}\left(\sum_{k=1}^{d}\left|u_{k}\right|\right)^{2}\right)\exp\left(-\frac{\sqrt{5}}{\ell}\sum_{k=1}^{d}\left|u_{k}\right|\right).\label{eq:matern_l1_52}
\end{align}
Using equations~\eqref{eq:k52_decomposition}-\eqref{eq:k52_decomposition_phi} and Proposition~\ref{prop:L1_kernel_decomposition}, the kernel sum $\sum_{i=1}^{N}y_{i}K_{5/2;\varsigma,\ell}(\mathbf{x}_{i}-\mathbf{z})$ admits the decomposition
\begin{align}
 & \sum_{i=1}^{N}y_{i}K_{5/2;\varsigma,\ell}(\mathbf{x}_{i}-\mathbf{z})\nonumber \\
 & = \  \varsigma^{2}\!\!\sum_{\boldsymbol{\delta}\in\{-1,1\}^{d}}\exp\!\left(-\frac{\sqrt{5}}{\ell}\boldsymbol{\delta.z}\right)\left[\left(1+\frac{\sqrt{5}}{\ell}\boldsymbol{\delta.z}+\frac{5}{3\ell^{2}}\left(\boldsymbol{\delta.z}\right)^{2}\right)F(\boldsymbol{\delta}\boldsymbol{z},\boldsymbol{\delta};\boldsymbol{\delta}\mathbf{x},\boldsymbol{w}^{(0)})\right.\nonumber \\
 & \ \ \ \ \left.-\left(\frac{\sqrt{5}}{\ell}+2\frac{5}{3\ell^{2}}\boldsymbol{\delta.z}\right)F(\boldsymbol{\delta}\boldsymbol{z},\boldsymbol{\delta};\boldsymbol{\delta}\mathbf{x},\boldsymbol{w}^{(1)})+\frac{5}{3\ell^{2}}F(\boldsymbol{\delta}\boldsymbol{z},\boldsymbol{\delta};\boldsymbol{\delta}\mathbf{x},\boldsymbol{w}^{(2)})\right]\label{eq:matern_52_decomposition}
\end{align}
and the gradient with respect to the lengthscale $\ell$ can be shown to be equal to
\begin{align*}
 & \frac{\partial}{\partial\ell}\left(\sum_{i=1}^{N}y_{i}K_{5/2;\varsigma,\ell}(\mathbf{x}_{i}-\mathbf{z})\right)=\sum_{i=1}^{N}\frac{5}{3}y_{i}\frac{\varsigma^{2}}{\ell^{3}}\left(\left\Vert \mathbf{x}_{i}-\mathbf{z}\right\Vert _{1}^{2}+\frac{\sqrt{5}}{\ell}\left\Vert \mathbf{x}_{i}-\mathbf{z}\right\Vert _{1}^{3}\right)\exp\left(-\frac{\sqrt{5}}{\ell}\left\Vert \mathbf{x}_{i}-\mathbf{z}\right\Vert _{1}\right)\\
 & =\frac{5}{3}\frac{\varsigma^{2}}{\ell^{3}}\sum_{\boldsymbol{\delta}\in\{-1,1\}^{d}}\exp\left(-\frac{\sqrt{5}}{\ell}\boldsymbol{\delta.z}\right)\left[( \frac{\sqrt{5}}{\ell}\left(\boldsymbol{\delta.z}\right)^{3} + 
 \left(\boldsymbol{\delta.z}\right)^{2})F(\boldsymbol{\delta}\boldsymbol{z},\boldsymbol{\delta};\boldsymbol{\delta}\mathbf{x},\boldsymbol{w}^{(0)})\right.\\
 & \quad  -\left.(2 (\delta.z)+ 3 \frac{\sqrt{5}}{\ell} \left(\boldsymbol{\delta.z}\right)^{2})F(\boldsymbol{\delta}\boldsymbol{z},\boldsymbol{\delta};\boldsymbol{\delta}\mathbf{x},\boldsymbol{w}^{(1)})+(3 \frac{\sqrt{5}}{\ell} \left(\boldsymbol{\delta.z}\right) +1) F(\boldsymbol{\delta}\boldsymbol{z},\boldsymbol{\delta};\boldsymbol{\delta}\mathbf{x},\boldsymbol{w}^{(2)})- \right.  \\
 &  \quad \quad \left.\frac{\sqrt{5}}{\ell}F(\boldsymbol{\delta}\boldsymbol{z},\boldsymbol{\delta};\boldsymbol{\delta}\mathbf{x},\boldsymbol{w}^{(3)})\right]\ .
\end{align*}
\end{example}

\section{Numerical results} \label{sec:numericalResults}

In this section, we begin by solving, on simulated data, the problem \eqref{eq:gp_log_likelihood} in dimensions one to three when the parameters $\bm{\theta}$ of the models are the scaling parameters defined in equation~\eqref{eq:scaled_kernel}. In a second part, we solve, on simulated data once again, the global problem \eqref{eq:gp_f} where both $m$ and $\sigma$ have to be estimated.
We detail the algorithm used and show that it converges correctly. %Finally, we provide some results when applied to a real dataset, showing that the estimation of $m$ is hard as the data does not follow our prior model, and we propose to stop the previous algorithm after one iteration for estimating $m$. 
In all the tests, the Mat\'ern-1/2 kernel is used.\\

To optimize in \eqref{eq:gp_log_likelihood} the parameters $\bm{\theta}$ of the models defined in \eqref{eq:scaled_kernel}, we classically use iterative methods. The most commonly used are LBFGS (we used the LBFGS++ library \cite{LBFGSpp}) and the ADAM algorithm \cite{kingma2014adam}. We found that in many examples the LBFGS algorithm fails in the line search due to a bad estimation of the descent direction. Moreover, the LBFGS algorithm not only relies on an estimate of the gradient, but also needs to evaluate the value of the objective function. We found that this evaluation fails when the parameters are too high or too small, due to the estimation of the $\log$ of the determinant : the Lanczos algorithm is known to be unstable because the Lanczos vectors form an orthogonal basis in theory, and this orthogonality can quickly be lost in practice, resulting in undefined values. Some re-orthogonalization could be done, see for example \cite{parlett1979lanczos}, but for simplicity, we decided to use the ADAM algorithm for all optimizations.\\
By default the initial learning rate is taken equal to $0.005$, decreasing linearly with the number of iterations to $0.0005$ for $20000$ iterations. We stop when the gradient norm is less than $10^{-3}$. Moreover, the function to be optimized has a very steep gradient near the optimum: during iterations, the gradient method can iterate a lot to stabilize, but the estimated parameters hardly change. Then, when the parameters show a change in norm less than $10^{-4}$, separated by one hundred gradient iterations, the process is stopped. During the conjugate iteration, we use a convergence criterion equal to $10^{-5}$, an incomplete Cholesky preconditioner with rank 100.
\begin{rem}
 The product matrix vector is not multithreaded: the running time could certainly be reduced further, even though divide-and-conquer-type algorithms are known to be difficult to parallelize. 
\end{rem}

\subsection{$d$-dimensional tests on simulated data} \label{subsection:simulatedDataDimD}
In this example derived from \cite{gyger2024iterative}, we optimize \eqref{eq:gp_log_likelihood} by taking the scaling parameters 
$\varsigma =1$, $\ell=0.1054$, and the nugget term $\sigma=1$.
The $\mathbf{x}$ are sampled uniformly on $[0,1]^d$, and the $y$ follow equation~\eqref{eq:gp_f} using $N$ samples drawn from the Mat\'ern-1/2 kernel.
Our iterative algorithm takes as initial values $\varsigma=0.5$ and $\ell=1$.

In  the tables  \ref{tab:sim1D}, \ref{tab:sim2D} and \ref{tab:sim3D}, we give the parameters estimated and the number of iterations used to converge. The time in seconds is given with an old Processor Intel Xeon® Gold 6234  (2019).

\begin{table}[H]
\noindent \begin{centering}
\begin{tabular}{lllll}
\toprule 
\multirow{2}{*}{\enskip{}$N$} & \multirow{2}{*}{\enskip{}Time} & \multicolumn{1}{l}{Number of} & \multicolumn{2}{l}{Parameters}\tabularnewline
 &  & iterations & $\varsigma$ & $\ell$\tabularnewline
\midrule 
\enskip{}20,000 & \enskip{}1956 & 6644 & 1.2048 & 0.1647\enskip{}\tabularnewline
\enskip{}40,000 & \enskip{}3655 & 7219 & 1.2436 & 0.1696\enskip{}\tabularnewline
100,000 & 20793 & 5033 & 0.6973 & 0.0514\enskip{}\tabularnewline
200,000 & 32143 & 7496 & 1.0602 & 0.1249\enskip{}\tabularnewline
\bottomrule
\end{tabular}
\par\end{centering}
\noindent \centering{}\caption{Estimation of parameters on simulated data for $d=1$.\label{tab:sim1D}}
\end{table}

\begin{table}[H]
\noindent \begin{centering}
\begin{tabular}{lllll}
\toprule 
\multirow{2}{*}{\enskip{}$N$} & \multirow{2}{*}{\enskip{}Time} & \multicolumn{1}{l}{Number of} & \multicolumn{2}{l}{Parameters}\tabularnewline
 &  & iterations & $\varsigma$ & $\ell$\tabularnewline
\midrule 
\enskip{}20,000 & \enskip{}1964 & 1601 & 0.9967 & 0.1098\enskip{}\tabularnewline
\enskip{}40,000 & 10211 & 1779 & 0.9964 & 0.1070\enskip{}\tabularnewline
100,000 & 22709 & 2176 & 0.9658 & 0.0973\enskip{}\tabularnewline
200,000 & 73818 & 2294 & 1.0062 & 0.1059\enskip{}\tabularnewline
\bottomrule
\end{tabular}
\par\end{centering}
\noindent \centering{}\caption{Estimation of parameters on simulated data for $d=2$.\label{tab:sim2D}}
\end{table}

\begin{table}[H]
\noindent \begin{centering}
\begin{tabular}{lllll}
\toprule 
\multirow{2}{*}{\enskip{}$N$} & \multirow{2}{*}{\enskip{}\enskip{}Time} & \multicolumn{1}{l}{Number of} & \multicolumn{2}{l}{Parameters}\tabularnewline
 &  & iterations & $\varsigma$ & $\ell$\tabularnewline
\midrule 
\enskip{}20,000 & \enskip{}\enskip{}6543 & \enskip{}750 & 1.0278 & 0.1107\enskip{}\tabularnewline
\enskip{}40,000 & \enskip{}21657 & \enskip{}918 & 0.9977 & 0.1108\enskip{}\tabularnewline
100,000 & 113826 & 1047 & 0.9994 & 0.1030\enskip{}\tabularnewline
\bottomrule
\end{tabular}
\par\end{centering}
\noindent \centering{}\caption{Estimation of parameters on simulated data for $d=3$.\label{tab:sim3D}}
\end{table}
We observe that the complexity of the algorithm clearly increases with the dimension, but surprisingly, on this example, the number of iterations needed decreases with the dimension, and the estimation of the parameters is more accurate.

\subsection{Join estimation of mean, scaling, and nugget parameters}
In this section our target is to solve the problem~\eqref{eq:gp_y}-\eqref{eq:gp_f} in dimension 2  using the dataset from \citet{gyger2024iterative}, assuming that $K$ in equation~\eqref{eq:gp_f} is a Mat\'ern-1/2 kernel with scaling  parameters $\varsigma>0$ and $\ell>0$,  as defined in equation~\eqref{eq:matern_l1_12}.

This data set is a series of temperatures $\{y_i\}_{i=1,\ldots,N}$ associated with some 2D coordinates $ \{\mathbf{x}_i = (x^W_i,x^N_i) \}_{i=1, \ldots, N}$ with  $N=400,000$. A point $\mathbf{x}$ is therefore defined  by   the longitude $x^W$ and the latitude $x^N$ of the temperature readings. We assume that the side effect function $m$ is affine:
\[
    m(\mathbf{x}) = \beta_0 + \beta_W x^W + \beta_N x^N.
\]
 where $\boldsymbol{\beta}= \{\beta_0 , \beta_W, \beta_N \}$ are three side effect parameters to be estimated. \\

The estimation of fixed effects, which is generally ignored in research articles on the subject, can be done by likelihood maximization while also estimating the kernel parameters. This method is used by \citet{guinness2021gaussian} or \citet{gyger2024iterative}, who developed a likelihood function analogous to our equation~\eqref{eq:gp_log_likelihood} that also incorporates the fixed effects coefficients. \citet{sigrist2022gaussian} uses this method, but also achieves good performance with a two-step procedure that iterates over the estimation of the fixed effects and the estimation of the kernel parameters. Here we first use ordinary least squares to get an initial estimate of $\boldsymbol{\beta}$, then we remove the estimated fixed effects $\hat{m}(\mathbf{x})$ from the initial data and proceed with estimating the kernel parameters on $\bar{y}=y - \hat{m}(\mathbf{x})$.
Since the nugget parameter $\sigma$ is unknown, we need to estimate it. To do so, we proceed as in \citet{sigrist2022gaussian}: we perform a gradient descent on the parameters $\varsigma$ and $\ell$, and at each step of the descent we update the estimated value of $\sigma$, as summarized in Algorithm~\ref{algo:gradientDescentSigma}.\par \medskip

As can be seen in Algorithm~\ref{algo:gradientDescentSigma}, one repeatedly computes the gradient and updates the estimates of $\varsigma$ and $\ell$, also obtaining the norm of the gradient. Then a matrix $\widetilde{K}$ is introduced, which is associated with the parameters $\frac{\varsigma_i}{\sigma_{i-1}}$ and $\ell_i$, so that
\[
    \mathbf{K}+\sigma^{2}\mathbf{I} = \sigma^2\left(\widetilde{\mathbf{K}} + \mathbf{I}\right),
\]
makes us interpret $\widetilde{\mathbf{K}}+\mathbf{I}$ as a covariance matrix associated with $\sigma=1$.
 Then the formula
\[
    \widetilde{\mathbf{K}} \longleftarrow \mathbf{K}\left(\frac{\varsigma_i}{\sigma_{i-1}}, \ell_i\right)
\]
used in Algorithm~\ref{algo:gradientDescentSigma} and
explained by \citet{sigrist2022gaussian} provides a closed-form expression for updating the estimate of $\sigma$ at each time step.
Once the gradient descent has converged, one uses the Generalized Least Squares estimator of $\boldsymbol{\beta}$, as defined for example in \cite{ludkovski2025gaussian}:
\[
    \hat{\beta}_{\mathrm{GLS}} = \left(\mathbf{H}^\top\left(\widehat{\mathbf{K}}+\hat{\sigma}^2\mathbf{I}\right)^{-1}\mathbf{H}\right)^{-1}\mathbf{H}^\top\left(\widehat{\mathbf{K}}+\hat{\sigma}^2\mathbf{I}\right)^{-1}y,
\]
where $\widehat{\mathbf{K}}$ is the covariance matrix computed with the estimated parameters $\hat{\varsigma}$ and $\hat{\ell}$, and $\mathbf{H}$ is the $N\times 3$ matrix with first a column of ones and then two columns with the West and then North coordinates of the points. We then repetitively perform the gradient descent and the update of $\hat{\beta}^{\mathrm{GLS}}$ until stability of $\hat{\beta}^{\mathrm{GLS}}$ is observed.\\

\begin{algorithm}[H]
{\scalefont{0.9}
    \caption{Successive gradient descents with control of the\\ successive $\beta$ coefficients and updates of the estimate for $\sigma$\label{algo:gradientDescentSigma}}
    \DontPrintSemicolon
    \KwIn{$y$ the vector of temperatures, size $N$ \\
          \hspace{3.4em}{\scalefont{0.93}$H$ the $(N,3)$ matrix with a column of ones and two columns of longitudes and latitudes} \\
          \hspace{3.2em} $\sigma_0$ an initial guess for the value of $\sigma$ \\
          \hspace{3.2em} $\varsigma_0$ an initial guess for the value of $\varsigma$ \\
          \hspace{3.2em} $\ell_0$ an initial guess for the value of $\ell$ \\
          \hspace{3.2em} $\varepsilon$ the size of the gradient to stop the optimization \\
          \hspace{3.2em} $\varepsilon_\beta$ the maximal size of the difference between successive values of $\beta$}
    \KwOut{$\hat{\sigma}$ an estimate of $\sigma$ \\
           \hspace{4.2em} $\hat{\varsigma}$ an estimate of $\varsigma$ \\
           \hspace{4.2em} $\hat{\ell}$ an estimate of $\ell$}
    \BlankLine
    $\hat{\beta}^{\mathrm{OLS}} \longleftarrow (H^\top H)^{-1} H^\top Y$ \tcp{OLS estimate}
    $y^{\mathrm{OLS}} \longleftarrow y-H\hat{\beta}^{\mathrm{OLS}}$ \tcp{Centered version of $y$ following OLS}
    $i\longleftarrow 0, j\longleftarrow 0$, $\varsigma_i^{0}\longleftarrow\varsigma_0$, $\sigma_i^{0}\longleftarrow\sigma_0$, $\ell_i^{0}\longleftarrow\ell_0$\;
    \Do{$\mathrm{gradNorm} > \varepsilon$}{\tcp{Gradient descent with $\beta$ issued from OLS}
        $j \longleftarrow j+1$\;
        $\varsigma_i^j,\ell_i^j,\mathrm{gradNorm} \longleftarrow \mathrm{GradDescStep}(y^{\mathrm{OLS}};\varsigma_i^{j-1},\ell_i^{j-1},\sigma_i^{j-1})$\;
        $\widetilde{\mathbf{K}} \longleftarrow \mathbf{K}(\frac{\varsigma_i^j}{\sigma_i^{j-1}}, \ell_i^j)$\;
        $(\sigma_i^j)^2 \longleftarrow \frac{1}{N}(y^{\mathrm{OLS}})^\top \left(\widetilde{\mathbf{K}}+\mathbf{I}\right)^{-1}y^{\mathrm{OLS}}$ \tcp{Update of $\sigma$, following \citet{sigrist2022gaussian}}
    }
    $\beta^{\mathrm{AVG}} \longleftarrow (0 \quad 0 \quad 0)^\top$\;
    $\beta^{\mathrm{old}} \longleftarrow \beta^{\mathrm{OLS}}$\;
    \Do{$\mathrm{changeInParam} > \varepsilon_\beta$}{\tcp{Repeat gradient descents until successive $\beta^{\mathrm{GLS}}$ get close}
        $\beta^{\mathrm{GLS}}\longleftarrow (H^\top [\mathbf{K}+(\sigma_i^j)^2\mathbf{I})]^{-1}H)^{-1}H^\top [\mathbf{K}+(\sigma_i^j)^2\mathbf{I}]^{-1}y$\;
        $\varsigma_{i+1}^0\longleftarrow\varsigma_i^j$, $\sigma_{i+1}^0\longleftarrow\sigma_i^j$, $\ell_{i+1}^0\longleftarrow\ell_i^j$\;
        $i\longleftarrow i+1$, $j\longleftarrow 0$\;
        $\beta^{\mathrm{AVG}} \longleftarrow \frac{i-1}{i} \beta^{\mathrm{AVG}} + \frac{1}{i} \beta^{\mathrm{GLS}}$ \tcp{Average of previous $\beta$}
        $y^{\mathrm{AVG}} \longleftarrow y-H\hat{\beta}^{\mathrm{AVG}}$\;
        \Do{$\mathrm{gradNorm} > \varepsilon$}{\tcp{Gradient descent with current $\beta^{\mathrm{AVG}}$}
            $j \longleftarrow j+1$\;
            $\varsigma_i^j,\ell_i^j,\mathrm{gradNorm} \longleftarrow \mathrm{GradDescStep}(y^{\mathrm{AVG}};\varsigma_i^{j-1},\ell_i^{j-1},\sigma_i^{j-1})$\;
            $\widetilde{\mathbf{K}} \longleftarrow \mathbf{K}(\frac{\varsigma_i^j}{\sigma_i^{j-1}}, \ell_i^j)$\;
            $(\sigma_i^j)^2 \longleftarrow \frac{1}{N}(y^{\mathrm{AVG}})^\top \left(\widetilde{\mathbf{K}}+\mathbf{I}\right)^{-1}y^{\mathrm{AVG}}$
        }
        $\mathrm{changeInParam} = \|\beta^{\mathrm{GLS}}-\beta^{\mathrm{old}}\|$\;
        $\beta^{\mathrm{old}} = \beta^{\mathrm{GLS}}$\;
    }
    $\hat{\sigma} \longleftarrow \sigma_i^j$\;
    $\hat{\varsigma} \longleftarrow \varsigma_i^j$\;
    $\hat{\ell} \longleftarrow \ell_i^j$\;
    $ \hat \beta^{\mathrm{GLS}} \longleftarrow \beta^{\mathrm{GLS}}$\;
    \Return $ \hat \beta^{\mathrm{GLS}}, \hat{\sigma}, \hat{\varsigma}, \hat{\ell}$
}% end scalefont
\end{algorithm}

\vspace{10mm}
%\subsubsection{Testing Algorithm~\ref{algo:gradientDescentSigma}  on simulated data}
%\label{subsubsection:testSimulatedData}

We want to evaluate the numerical performance of our optimization algorithm designed to optimize the fixed effects, $\sigma$, and the scaling parameters. To do so, we simulate data with parameters similar to the real case and perform the estimation of $\{\boldsymbol{\beta}, \sigma , \varsigma, \ell \}$
using 40,000, then 100,000, and finally 200,000 data points. In each of these three cases, we get a stable $\hat{\beta}^{\mathrm{GLS}}$ after the third iteration of the gradient descent and computation of $\hat{\beta}^{\mathrm{GLS}}$. 
In Table~\ref{tab:simFromRealData} we report the parameters used for the simulation and the estimates we obtained. As we can see, the fixed effects and the nugget parameter are recovered quite accurately. The estimated values of $\varsigma$ and $\ell$ are further away from the real values used for the simulation.
%\pg{Je mettrais ``repetition'' plutôt que ``iteration'' pour bien suggérer qu'on parle de l'alternance descente / recalcul de $\beta$ et pas juste d'un pas de la descente de gradient.}

\begin{table}[H]
\begin{centering}
\renewcommand{\arraystretch}{1.2} %
\begin{tabular}{lllllll}
\toprule 
\multicolumn{1}{l}{\enskip{}Parameters} & $\beta_{0}${\scalefont{0.9} ($\tccentigrade$)} & \multicolumn{1}{l}{$\beta_{W}${\scalefont{0.9} ($\tccentigrade\cdot\text{m}^{-1}$)}} & \multirow{1}{*}{$\beta_{N}${\scalefont{0.9} ($\tccentigrade\cdot\text{m}^{-1}$)}} & $\sigma${\scalefont{0.9} ($\tccentigrade$)} & $\varsigma${\scalefont{0.9} ($\tccentigrade$)} & $\ell${\scalefont{0.9} ($\text{m}^{-1}$)}\enskip{}\tabularnewline
\midrule 
\multirow{2}{*}{\enskip{}\makecell[l]{Value for the\\ simulation}} & \multirow{2}{*}{-53.0} & \multirow{2}{*}{$-8.40\cdot10^{-6}$} & \multirow{2}{*}{$4.50\cdot10^{-6}$} & \multirow{2}{*}{1.60} & \multirow{2}{*}{0.200} & \multirow{2}{*}{$4.00\cdot10^{5}$\enskip{}}\tabularnewline
 &  &  &  &  &  & \tabularnewline
\multirow{2}{*}{\enskip{}\makecell[l]{Estimated value\\ 40,000 points}} & \multirow{2}{*}{-54.0} & \multirow{2}{*}{$-8.50\cdot10^{-6}$} & \multirow{2}{*}{$4.60\cdot10^{-6}$} & \multirow{2}{*}{1.60} & \multirow{2}{*}{0.132} & \multirow{2}{*}{$1.91\cdot10^{5}$\enskip{}}\tabularnewline
 &  &  &  &  &  & \tabularnewline
\multirow{2}{*}{\enskip{}\makecell[l]{Estimated value\\ 100,000 points}} & \multirow{2}{*}{-53.0} & \multirow{2}{*}{$-8.41\cdot10^{-6}$} & \multirow{2}{*}{$4.46\cdot10^{-6}$} & \multirow{2}{*}{1.60} & \multirow{2}{*}{0.190} & \multirow{2}{*}{$2.43\cdot10^{5}$\enskip{}}\tabularnewline
 &  &  &  &  &  & \tabularnewline
\multirow{2}{*}{\enskip{}\makecell[l]{Estimated value\\ 200,000 points}} & \multirow{2}{*}{-54.5} & \multirow{2}{*}{$-8.64\cdot10^{-6}$} & \multirow{2}{*}{$4.45\cdot10^{-6}$} & \multirow{2}{*}{1.61} & \multirow{2}{*}{0.163} & \multirow{2}{*}{$3.51\cdot10^{5}$\enskip{}}\tabularnewline
 &  &  &  &  &  & \tabularnewline
\bottomrule
\end{tabular}
\par\end{centering}
\noindent \centering{}\caption{Parameters used for the simulation and values taken by the estimators.\label{tab:simFromRealData}}
\end{table}

We also examine the sensitivity of the gradient descent results to the initial values, the learning rate, and the maximum number of iterations of the ADAM optimizer (which affects the decreasing speed of the learning rate). To do this, we fix $\boldsymbol{\beta}$ and $\sigma$ to their true values --- because we have shown above that they are quite easy to estimate accurately. We only try to estimate $\varsigma$ and $\ell$. We perform their estimation many times with different settings using the first 40,000 lines of the sample. Our results are summarized in Table~\ref{tab:fixBetaSigma} and show that in our sample, the values taken by the estimators are independent of the starting point and stable with the different learning rates that we try.

\begin{table}[H]
{\scalefont{0.98}
    \centering
    \begin{tabular}{llllll}
        \toprule
        Starting & Initial lear- & Minimum & Maximum num- & Estimated & Estimated \\
        point & ning rate & learning rate & ber of iterations & $\varsigma${\scalefont{0.9} ($\tccentigrade$)} & $\ell${\scalefont{0.9} ($\text{m}^{-1}$)} \\
        \midrule
        $(0.35, 0.125)$ & $1\cdot 10^{-2}$ & $1\cdot 10^{-4}$ & 10,000 & 0.150 & $2.79\cdot 10^5$  \\
        $(0.35, 0.125)$ & $5\cdot 10^{-2}$ & $5\cdot 10^{-4}$ & 10,000 & 0.150 & $2.79\cdot 10^5$ \\
        $(0.35, 0.125)$ & $5\cdot 10^{-2}$ & $5\cdot 10^{-4}$ & \enskip{}1,000 & 0.150 & $2.79\cdot 10^5$ \\
        $(0.35, 0.125)$ & $5\cdot 10^{-2}$ & $5\cdot 10^{-4}$ & \enskip{}5,000 & 0.150 & $2.79\cdot 10^5$ \\
        $(0.85, 0.625)$ & $5\cdot 10^{-2}$ & $5\cdot 10^{-4}$ & 10,000 & 0.150 & $2.79\cdot 10^5$ \\
        \bottomrule
    \end{tabular}
    \caption{Estimation settings and results while $\boldsymbol{\beta}$ and $\sigma$ are fixed.\\ Parameters one should recover: $\varsigma = 0.200~\tccentigrade$, $\ell = 4.00 \cdot 10^5\text{~m}^{-1}$.}
    \label{tab:fixBetaSigma}
}% end scalefont
\end{table}

\section{Conclusion\label{sec:conclusionsec}}
We developed an exact fast kernel matrix-vector multiplication (MVM) algorithm, based on exact kernel decomposition into weighted empirical cumulative distribution functions, combined with fast multivariate CDF computation \cite{langrene2021fast}, and used it to speed up the optimization of the parameters of Gaussian process regression models.
This method, compatible with various types of kernels (those satisfying Assumption~\ref{assu:kernel_decomposition_multivariate}), is very well suited for ``large $N$, small $d$'' situations, where $N$ is the number of data points, and $d$ is the number of features of each data point.\\
The algorithm's reduced computational complexity and memory complexity compared to direct MVM allowed us to optimize models with large datasets consisting of hundreds of thousands of points in dimensions one to three. Our tests focused on Mat\'ern kernels, for which we provided explicit multivariate kernel decomposition formulas (equations \eqref{eq:matern_12_decomposition}-\eqref{eq:matern_32_decomposition}-\eqref{eq:matern_product_32_decomposition}-\eqref{eq:matern_52_decomposition}).\\
Our successful results obtained on datasets with hundreds of thousands of points rely on a sequential implementation without GPU acceleration. Further computational improvements could be obtained in the future by developing a parallel version of our fast CDF computation algorithm, opening the door to exact Gaussian process inference on datasets consisting of millions of points or more.

\section*{Acknowledgements}

The authors would like to thank Christian Walder (Google DeepMind) and Tim Gyger (University of Zurich) for fruitful discussions. Nicolas Langren\'e acknowledges the partial support of the Guangdong Provincial/Zhuhai Key Laboratory IRADS (2022B1212010006). Pierre Gruet and Xavier Warin acknowledge support from the FiME Lab.

\bibliographystyle{apalike}
\bibliography{biblio}

\appendix

\section{Computational aspects of likelihood optimization\label{sec:likelihood_optimization}}

In the introduction, we explain that the covariance kernel depends on some hyperparameters $\bm{\theta}$ which have to be calibrated on the dataset. We do so by maximizing the log-marginal likelihood~\eqref{eq:gp_log_likelihood}. In this section we highlight the main steps of the computation of the log-marginal likelihood or of its gradient with respect to $\bm{\theta}$, given by equation~\eqref{eq:gp_log_likelihood_derivative}. As highlighted in \citet{gardner2018gpytorch}, the computationally expensive tasks for computing the log-marginal likelihood and its gradient are (1) linear solves $(\mathbf{K}+\sigma^{2}\mathbf{I})^{-1}\mathbf{y}$, (2) computations of logarithms of determinants $\mathrm{\log}(\det(\mathbf{K}+\sigma^{2}\mathbf{I}))$, and (3) computations of trace terms $\mathrm{tr}\left((\mathbf{K}+\sigma^{2}\mathbf{I})^{-1}\frac{\partial\mathbf{K}}{\partial\bm{\theta}}\right)$: (1) is computed using conjugate gradients, (2) is obtained from a \citet{lanczos1950iteration} tridiagonalization procedure associated with conjugate gradients, and (3) is obtained using stochastic trace estimation. First we discuss the estimation steps with no preconditioning, and then we explain how to adapt the method when introducing a preconditioner in the computations.\par
The conjugate gradient algorithm is a popular method to solve a system of linear equations associated with a positive-definite symmetric matrix. It is presented e.g. in \citet{saad2003iterative}. Although it converges in $n$ iterations where $n$ is the size of its matrix, an appropriate preconditioning step can greatly reduce the number of required iterations. To compute $(\mathbf{K}+\sigma^{2}\mathbf{I})^{-1}\mathbf{y}$, the conjugate gradient algorithm does not need $(\mathbf{K}+\sigma^{2}\mathbf{I})$ itself, but only a function that computes matrix-vector products featuring this matrix. In practice, we use the fast matrix-vector multiplication algorithms introduced in Section~\ref{sec:kernel_mvm}.

\subsection{Conjugate gradient with multiple second members and partial Lanczos tridiagonalization to compute the logarithms of the determinants}

As was done in \citet{gardner2018gpytorch}, we use a modified version of the conjugate gradient algorithm that takes $\ell \geq 1$ second members $\mathbf{z}_1, ..., \mathbf{z}_\ell$ as inputs and also computes several partial Lanczos tridiagonal matrices associated with $(\mathbf{K}+\sigma^{2}\mathbf{I})$. Lanczos tridiagonalization is a procedure, named after \citet{lanczos1950iteration}, which consists in writing the approximation $VTV^T$ of a Hermitian $n\times n$ matrix, where $T$ is a tridiagonal real $m \times m$ matrix and $V$ is a $n\times m$ matrix with orthonormal columns. If $m=n$ then the decomposition is exact. Here we perform only a partial tridiagonalization by using the links between conjugate gradient and Lanczos decomposition, which are discussed in \citet[Section 6.7]{saad2003iterative}: one will perform $m$ iterations of the conjugate gradient and get $\ell$ times $m\times m$ Lanczos tridiagonal matrices as well as the approximate solutions $(\mathbf{K}+\sigma^{2}\mathbf{I})^{-1} \mathbf{z}_1, ..., (\mathbf{K}+\sigma^{2}\mathbf{I})^{-1} \mathbf{z}_\ell$. Equation~(6.82) in \citet{saad2003iterative} shows how to compute Lanczos tridiagonal coefficients using the conjugate gradient steps, and Algorithm~2 in the paper supplement of \citet{gardner2018gpytorch} exhibits a possible implementation of the modified conjugate gradient algorithm. Getting the Lanczos coefficients from the conjugate gradient steps avoids facing the numerical instabilities associated with the sole Lanczos algorithm and reported e.g. by \citet{saad2003iterative}.\par

\begin{algorithm}[H]
    \caption{Modified conjugate gradient algorithm with preconditioning\label{algo:MCGLanczos}}
    \DontPrintSemicolon
    \KwIn{$f_M$ a function to multiply $M$ by a matrix at right \\
          \hspace{3.2em} $B$ a $N\times \ell$ matrix \\
          \hspace{3.2em} $\bar{n}_I$ a maximal number of iterations to perform \\
          \hspace{3.2em} $\mathrm{tol}$ the tolerance on the norm of the residuals one aims at reaching \\
          \hspace{3.2em} $M_B^0$ an initial guess about the value of $M^{-1}B$ \\
          \hspace{3.2em} $s_P$ a function to multiply the inverse of the preconditioner $P$ by a matrix at right}
    \KwOut{$M_B$ an approximation of $M^{-1}B$ \\
           \hspace{4.2em} $n_I$ the realized number of iterations \\
           \hspace{4.2em} $T_1, ..., T_\ell$ tridiag. matrices from partial Lanczos decompositions of $M$}
    \BlankLine
    $R \longleftarrow B - f_M(M_B^0)$ \tcp{Initial residual}
    $\overline{P} \longleftarrow s_P(R)$ \tcp{Initial search direction}
    \For{$j \longleftarrow 1$ \KwTo $\ell$}{
        $k[j] \longleftarrow \overline{P}[\cdot, j]^T R[\cdot, j]$
    }
    $n_I \longleftarrow 0$\;
    $M_B \longleftarrow M_B^0$\;
    $T_1, ..., T_\ell \longleftarrow \mathbf{0}_{\bar{n}_I, \bar{n}_I}$\;
    $\alpha_\mathrm{prev} \longleftarrow \mathbf{1}_\ell$\;
    $\beta_\mathrm{prev} \longleftarrow \mathbf{0}_\ell$
    \BlankLine
    \While{$n_I < \bar{n}_I$}{
        $T \longleftarrow f_M\left(\overline{P}\right)$\;
        \For{$j \longleftarrow 1$ \KwTo $\ell$}{
            $\alpha[j] \longleftarrow k[j] \; / \; \left(\overline{P}[\cdot, j]^T T[\cdot, j]\right)$\;
            $M_B[\cdot, j] \longleftarrow M_B[\cdot, j] + \alpha[j] \overline{P}[\cdot, j] $\;
            $R[\cdot, j] \longleftarrow R[\cdot, j] - \alpha[j] T[\cdot, j]$
        }
        \If{$\|R\|^2 < \mathrm{tol}$}{
            \Return $M_B, n_I, T_1, ..., T_\ell$
        }
        $Z \longleftarrow s_P(R)$ \tcp{Solve $PZ = R$}
        $k_\mathrm{prev} \longleftarrow k$\;
        $n_I \longleftarrow n_I + 1$\;
        \For{$j \longleftarrow 1$ \KwTo $\ell$}{
            $k[j] \longleftarrow Z[\cdot, j]^T R[\cdot, j]$\;
            $\beta[j] \longleftarrow k[j] / k_\mathrm{prev}[j]$\;
            $\overline{P}[\cdot, j] \longleftarrow Z[\cdot, j] + \beta[j]\overline{P}[\cdot, j]$\;
            $T_j[n_I, n_I] \longleftarrow 1 / \alpha[j] + \beta_\mathrm{prev}[j] / \alpha_\mathrm{prev}[j]$\;
            $T_j[n_I + 1, n_I], T_j[n_I, n_I + 1] \longleftarrow \sqrt{\beta[j]} / \alpha[j]$
        }
        $\alpha_\mathrm{prev} \longleftarrow \alpha$\;
        $\beta_\mathrm{prev} \longleftarrow \beta$
    }
    \Return $M_B, n_I, T_1, ..., T_\ell$
\end{algorithm}

Algorithm~\ref{algo:MCGLanczos} details the steps of the conjugate gradient with multiple second members and the computation of Lanczos tridiagonal matrices. This algorithm is used in many places in the code as several linear solves have to be done, but Lanczos matrices are computed only to get the logarithm of the determinant appearing in the log-marginal likelihood. When $l$ Lanczos tridiagonal matrices of size $m$ have to be computed, $m$ iterations of Algorithm~\ref{algo:MCGLanczos} with a null initial guess (that is, $M_B^0 = \mathbf{0}_{N\times \ell}$) are used to guarantee that the Lanczos decompositions are performed with the vectors $\mathbf{z}_1, ..., \mathbf{z}_\ell$ as initial vectors, see \citet{saad2003iterative}.

The details of the computation of the logarithms of determinants, also performed by \citet{gardner2018gpytorch} and \citet{gyger2024iterative}, is as follows: let us choose initial probe vectors $\mathbf{z}_1, ..., \mathbf{z}_\ell$ in $\mathbb{R}^N$ that are independently sampled according to the same Gaussian law with mean $0$ and variance $\mathbf{I}$, the identity matrix of size $N\times N$. Let us also use the null vector of $\mathbf{R}^N$ as the initial solution guess in the conjugate gradient algorithm. We have 
\[
    \mathrm{\log}(\det(\mathbf{K}+\sigma^{2}\mathbf{I})) = \mathrm{tr}(\mathrm{\log}(\mathbf{K}+\sigma^{2}\mathbf{I})),
\]
where $\log(M)$ is the matrix logarithm of the matrix $M$. 
We denote $T_i$, $i=1,...,\ell$ the tridiagonal matrices obtained from  the partial Lanczos decompositions $V_i T_i V_i^T$ of $(\mathbf{K}+\sigma^{2}\mathbf{I})$ of size $m<N$ where the algorithm is initialized with $\mathbf{z}_i$.\\
We use the stochastic trace estimator of \citet{hutchinson1989stochastic} to get
\begin{align*}
    \mathrm{\log}(\det(\mathbf{K}+\sigma^{2}\mathbf{I})) &= \mathbb{E}(\mathbf{z}_1^T \mathrm{\log}(\mathbf{K}+\sigma^{2}\mathbf{I}) \mathbf{z}_1) \\
    & \approx \frac{1}{\ell} \sum_{i=1}^\ell \mathbf{z}_i^T \mathrm{\log}(\mathbf{K}+\sigma^{2}\mathbf{I}) \mathbf{z}_i \\
    & \approx \frac{1}{\ell} \sum_{i=1}^\ell \mathbf{z}_i^T V_i\mathrm{\log}(T_i)V_i^T \mathbf{z}_i
\end{align*}
featuring also the Monte Carlo approximation first, and then the partial Lanczos decompositions. 
In order to further approximate the logarithm of the determinant, we underline the fact that

\[
    V_i^T \mathbf{z}_i = \|\mathbf{z}_i\|_2 \mathbf{e}_1 \approx \sqrt{N} \mathbf{e}_1,
\]
where $\|\mathbf{z}_i\|_2\approx \sqrt{N}$ is the $L^2$-norm of $\mathbf{z}_i$ and $\mathbf{e}_1$ is the vector $(1 \; 0 \; \cdots \; 0)^T$ of size $N$. The equality holds because the columns of the matrices $V_i$ are orthonormal and, by design of the Lanczos algorithm, the first column of $V_i$ is equal to the normalized probe vector if one uses a null initial guess in Algorithm~\ref{algo:MCGLanczos}. Thus it follows that
\begin{equation}
    \mathrm{\log}(\det(\mathbf{K}+\sigma^{2}\mathbf{I})) \approx \frac{N}{\ell} \sum_{i=1}^\ell \mathbf{e}_1^T\mathrm{\log}(T_i)\mathbf{e}_1,
    \label{eqn:approximationLogDet}
\end{equation}
and then the calculus of the matrix logarithms $\mathrm{\log}(T_i)$ is computationally efficient, see for instance \citet{coakley2013fast}.\\

\subsection{Stochastic trace estimation to compute the trace terms}

As in \citet{gardner2018gpytorch} and \citet{gyger2024iterative}, we use the stochastic trace estimator of \citet{hutchinson1989stochastic}, which is unbiased, to get
\[
    \mathrm{tr}\left((\mathbf{K}+\sigma^{2}\mathbf{I})^{-1}\frac{\partial\mathbf{K}}{\partial\bm{\theta}}\right) \approx \frac{1}{\ell} \sum_{i=1}^\ell \mathbf{z}_i^T (\mathbf{K}+\sigma^{2}\mathbf{I})^{-1}\frac{\partial\mathbf{K}}{\partial\bm{\theta}} \mathbf{z}_i,
\]
which features the linear solves $(\mathbf{K}+\sigma^{2}\mathbf{I})^{-1} z_i$ and the matrix-vector products $\frac{\partial\mathbf{K}}{\partial\bm{\theta}} \mathbf{z}_i$. The former are approximated by the outputs of the modified conjugate gradient described above, and the latter are obtained efficiently.

\subsection{Preconditioning}

In our computations we turned to preconditioning, which consists in introducing an invertible matrix $\mathbf{P}$ and solving $\mathbf{P}^{-1}(\mathbf{K}+\sigma^{2}\mathbf{I}) \mathbf{x} = \mathbf{P}^{-1}\mathbf{y}$ instead of $(\mathbf{K}+\sigma^{2}\mathbf{I}) \mathbf{x} = \mathbf{y}$. The convergence of the conjugate gradient may be quicker with a preconditioner as it depends on the conditioning of $\mathbf{P}^{-1}(\mathbf{K}+\sigma^{2}\mathbf{I})$ instead of the one of $\mathbf{K}+\sigma^{2}\mathbf{I}$. We use the pivoted Cholesky decomposition \cite{harbrecht2012low} to provide a preconditioner. We defer to \citet{gardner2018gpytorch} and \citet{gyger2024iterative} for the details about the computation and its efficiency. Their analysis underlines the cost of linear solves using the preconditioner is $\mathcal{O}(N^2k)$ if one chooses a preconditioner with an underlying Cholesky decomposition of rank $k\leq N$, that is, the preconditioner is written as $LL^T + \sigma^2 I_N$, where $L$ is a lower triangular $N\times k$ matrix and $L L^T$ is the rank $k$ pivoted Cholesky decomposition of $\mathbf{K}$.
\par \medskip

The computation of $\mathbf{P}^{-1}y$ is done thanks to the Woodbury formula \cite{woodbury1950inverting}, which yields
\[
    \mathbf{P}^{-1}y = \frac{1}{\sigma^2}y - \frac{1}{\sigma^4} L \left(I_k + \frac{1}{\sigma^2}L^T L\right)^{-1} L^T y.
\]

A side effect of using a preconditioner $\mathbf{P}$ is that the sum in Equation~\eqref{eqn:approximationLogDet} no longer yields $\mathrm{\log}(\det(\mathbf{K}+\sigma^{2}\mathbf{I}))$, but $\mathrm{\log}(\det(\mathbf{P}^{-1}(\mathbf{K}+\sigma^{2}\mathbf{I})))$ instead, because it is based on partial Lanczos decompositions of $\mathbf{P}^{-1}(\mathbf{K}+\sigma^{2}\mathbf{I})$. The logarithm of the determinant can be obtained as
\[
    \mathrm{\log}(\det(\mathbf{K}+\sigma^{2}\mathbf{I})) = \mathrm{\log}(\det(\mathbf{P}^{-1}(\mathbf{K}+\sigma^{2}\mathbf{I}))) + \mathrm{\log}(\det(\mathbf{P})).
\]
The first term of the sum is approximated using Equation~\eqref{eqn:approximationLogDet} while preconditioning, and the second term is obtained using Sylvester's determinant theorem:
\begin{align*}
    \det(L L^T + \sigma^2 I_N) &= \det\left(\sigma^2 I_N\left(\frac{1}{\sigma^2}L L^T + I_N\right)\right) \\
    &= \sigma^{2N}\det\left(\frac{1}{\sigma^2}L L^T + I_N\right) \\
    &= \sigma^{2N}\det\left(\frac{1}{\sigma^2}L^T L + I_k\right) 
\end{align*}
so that
\[
    \mathrm{\log}(\det(L L^T + \sigma^2 I_N)) =2N \mathrm{\log}(\sigma) + \mathrm{\log}\left(\det\left(\frac{1}{\sigma^2}L^T L + I_k\right)\right)
\]
which boils down to computing the determinant of a symmetric matrix of size $k\times k$.

\section{Positive definite multivariate kernels\label{sec:positive_definite}}

This appendix provides conditions under which the multivariate kernels
\eqref{eq:product_kernel} and \eqref{eq:L1_kernel} are positive
definite.
\begin{prop}
\label{prop:positive_definite_product_kernel}Suppose that the multivariate
isotropic kernel $K(\mathbf{u})=k(\left\Vert \mathbf{u}\right\Vert _{2})$,
$\mathbf{u}\in\mathbb{R}^{d}$, is positive definite for all $d\geq1$.
Then, the multivariate product kernel $K_{\Pi}(\mathbf{u})=\prod_{k=1}^{d}k(\left|u_{k}\right|)$
is also positive definite for all $d\geq1$.
\end{prop}
\begin{proof}
Since $K(\mathbf{u})=k(\left\Vert \mathbf{u}\right\Vert _{2})$ is
positive definite for all $d\geq1$, according to Corollary 1a) in
\citep{langrene2024mixture}, there exists a nonnegative random variable
$R$ such that its Laplace transform is equal to $k(\sqrt{.})/k(0)$:
\begin{equation}
k(s)=k(0)\mathbb{E}\left[e^{-Rs^{2}}\right],\ \forall s>0.\label{eq:Rs2}
\end{equation}
Then, define the random vector $\bm{\eta}:=\sqrt{R}\boldsymbol{S}_{2}$
where $\boldsymbol{S}_{2}=\sqrt{2}\boldsymbol{G}$ is a multivariate
Gaussian random vector with independent components with mean zero
and variance $2$. According to Corollary 1c) in \citep{langrene2024mixture}
with $\alpha=1$, the distribution of $\bm{\eta}$ is equal to the
spectral distribution of $K$:
\begin{equation}
K(\mathbf{u})=k(0)\mathbb{E}\left[e^{i\bm{\eta}^{\top}\mathbf{u}}\right]=k(0)\mathbb{E}\left[e^{i\sqrt{R}\boldsymbol{S}_{2}^{\top}\mathbf{u}}\right],\ \forall\mathbf{u}\in\mathbb{R}^{d}.\label{eq:eta_L2}
\end{equation}
Define the product kernel $K_{\Pi}(\mathbf{u})=\prod_{k=1}^{d}k(\left|u_{k}\right|)$,
and define the random vector $\bm{\eta}_{\Pi}:=\sqrt{\boldsymbol{R}}\boldsymbol{S}_{2}$,
where $\boldsymbol{R}=(R_{1},\ldots,R_{d})$ where the $R_{i}$ are
independent and identically distributed copies of $R$, independent
of $\boldsymbol{S}_{2}$. Then
\begin{align*}
k(0)\mathbb{E}\left[e^{i\bm{\eta}_{\Pi}^{\top}\mathbf{u}}\right] & =k(0)\mathbb{E}\left[e^{i(\sqrt{\boldsymbol{R}}\boldsymbol{S}_{2})^{T}\mathbf{u}}\right]=\prod_{k=1}^{d}k(0)\mathbb{E}\left[e^{i\sqrt{R}S_{2}u_{k}}\right]\\
 & =\prod_{k=1}^{d}K(u_{k})=\prod_{k=1}^{d}k(\left|u_{k}\right|)=K_{\Pi}(\mathbf{u})
\end{align*}
where we used the independence between the components of $\bm{\eta}_{\Pi}$,
and equation \eqref{eq:eta_L2} with $d=1$. This shows that the distribution
of $\bm{\eta}_{\Pi}$ is equal to the spectral distribution of $K_{\Pi}$.
According to Bochner's theorem, the existence of the spectral distribution
of $K_{\Pi}$ proves that the kernel $K_{\Pi}(\mathbf{u})=\prod_{k=1}^{d}k(\left|u_{k}\right|)$
is positive definite.
\end{proof}
\begin{defn}
\label{def:symmetric_stable}For any $\alpha\in(0,2]$, let $\boldsymbol{S}_{\alpha}$
be a $d$-dimensional random vector with characteristic function $\phi_{\alpha}$
given by
\begin{equation}
\phi_{\alpha}(\mathbf{u})=\mathbb{E}\left[e^{i\boldsymbol{S}_{\alpha}^{\top}\mathbf{u}}\right]=e^{-\left\Vert \mathbf{u}\right\Vert ^{\alpha}}\ ,\ \mathbf{u}\in\mathbb{R}^{d}\label{eq:symmetric_stable-1}
\end{equation}
where $\left\Vert \mathbf{u}\right\Vert =\sqrt{u_{1}^{2}+\ldots+u_{d}^{2}}$
is the Euclidean norm of $\mathbf{u}=(u_{1},\ldots,u_{d})\in\mathbb{R}^{d}$.
The vector $\boldsymbol{S}_{\alpha}$ is called \textit{symmetric
stable.}
\end{defn}
When $d=1$ and $\alpha=1$, the symmetric stable random variable
$S_{1}$ is such that
\begin{equation}
\mathbb{E}\left[e^{iS_{1}u_{1}}\right]=e^{-\left\Vert u_{1}\right\Vert }=e^{-\left|u_{1}\right|}\label{eq:symmetric_stable_1d}
\end{equation}

for any $u_{1}\in\mathbb{R}$. Now, define the random vector$ $
\begin{equation}
\boldsymbol{S}_{1}^{1}:=\left[S_{1}^{(1)},\ldots,S_{1}^{(d)}\right]^{\top}\label{eq:symmetric_stable_L1_vector}
\end{equation}
where the random variables $S_{1}^{(j)}$, $j=1,2,\ldots,d$, are
independent and identically distributed copies of $S_{1}$. Then,
for any $\mathbf{u}=(u_{1},\ldots,u_{d})\in\mathbb{R}^{d}$,
\begin{equation}
\mathbb{E}\left[e^{i\boldsymbol{S}_{1}^{1\top}\mathbf{u}}\right]=\prod_{j=1}^{d}\mathbb{E}\left[e^{iS_{1}^{(j)}u_{j}}\right]=\prod_{j=1}^{d}e^{-\left|u_{j}\right|}=e^{-\left\Vert \mathbf{u}\right\Vert _{1}}\label{eq:symmetric_stable_L1}
\end{equation}
where $\left\Vert \mathbf{u}\right\Vert _{1}=\left|u_{1}\right|+\ldots+\left|u_{d}\right|$
is the L1-norm of $\mathbf{u}$.
\begin{lem}
\label{lem:RS1}Let $R$ be a real-valued nonnegative random variable,
independent of $\boldsymbol{S}_{1}^{1}$, with Laplace transform $\mathcal{L}$.
Then, the random projection vector defined by
\begin{equation}
\boldsymbol{\eta}=R\boldsymbol{S}_{1}^{1}\label{eq:isotropic_random_projection}
\end{equation}
spans the following shift-invariant kernel $K:\mathbb{R}^{d}\rightarrow\mathbb{R}$:
\begin{equation}
K(\mathbf{u})=K(\mathbf{0})\mathbb{E}\left[e^{i\boldsymbol{\eta}^{\top}\mathbf{u}}\right]=K(\mathbf{0})\mathcal{\mathcal{L}}(\left\Vert \mathbf{u}\right\Vert _{1})\ ,\ \mathbf{u}\in\mathbb{R}^{d}.\label{eq:isotropic_kernel}
\end{equation}
\end{lem}
\begin{proof}
Recall that the Laplace transform $\mathcal{L}$ of a nonnegative
random variable $R$ is defined by $\mathcal{L}(s)=\mathbb{E}\left[e^{-sR}\right]$
for $s\geq0$. Then, the characteristic function of $\bm{\eta}$ is
given by
\begin{align*}
\mathbb{E}\left[e^{i\boldsymbol{\eta}^{\top}\mathbf{u}}\right] & =\mathbb{E}\left[\mathbb{E}\left[\exp\left(i\boldsymbol{S}_{1}^{1\top}\left(R\mathbf{u}\right)\right)\left|R\right.\right]\right]\\
 & =\mathbb{E}\left[\exp\left(-R\left\Vert \mathbf{u}\right\Vert _{1}\right)\right]\\
 & =\mathcal{\mathcal{L}}(\left\Vert \mathbf{u}\right\Vert _{1})
\end{align*}
which proves equation \eqref{eq:isotropic_kernel}.
\end{proof}
\begin{prop}
\label{prop:positive_definite_L1_kernel}Suppose that the multivariate
isotropic kernel $K(\mathbf{u})=k(\left\Vert \mathbf{u}\right\Vert _{2}^{2})$,
$\mathbf{u}\in\mathbb{R}^{d}$, is positive definite for all $d\geq1$.
Then, the multivariate L1 kernel $K_{1}(\mathbf{u})=k(\left\Vert \mathbf{u}\right\Vert _{1})$
is also positive definite for all $d\geq1$.
\end{prop}
\begin{proof}
Since $K(\mathbf{u}):=k(\left\Vert \mathbf{u}\right\Vert _{2}^{2})$
is positive definite for all $d\geq1$, according to Theorem 1a) in
\citep{langrene2024mixture}, there exists a nonnegative random variable
$R$ such that its Laplace transform is equal to $k(.)/k(0)$:
\begin{equation}
k(s)=k(0)\mathbb{E}\left[e^{-Rs}\right],\ \forall s>0.\label{eq:Rs}
\end{equation}

Define the L1 kernel $K_{1}(\mathbf{u})=k(\left\Vert \mathbf{u}\right\Vert _{1})$,
and define the random vector $\bm{\eta_{1}}:=R\boldsymbol{S}_{1}^{1}$,
where $\boldsymbol{S}_{1}^{1}$ is defined as in \eqref{eq:symmetric_stable_L1_vector}
and is independent of $R$. Then, according to Lemma~\ref{lem:RS1},
\[
k(0)\mathbb{E}\left[e^{i\boldsymbol{\eta_{1}}^{\top}\mathbf{u}}\right]=k(0)\mathcal{\mathcal{L}}(\left\Vert \mathbf{u}\right\Vert _{1})=k(\left\Vert \mathbf{u}\right\Vert _{1})=K_{1}(\mathbf{u})
\]
As in Proposition~\ref{prop:positive_definite_product_kernel}, the
fact that the spectral distribution of $K_{1}$ exists proves that
$K_{1}$ is positive definite.
\end{proof}

\section{Mat\'ern covariance functions\label{sec:matern}}

This appendix gathers all the univariate Mat\'ern
covariance formulas used in this article, as well as their derivatives. Multivariate Mat\'ern
covariance functions are constructed in Subsection \ref{subsec:multivariate-case}. For a thorough discussion
on the Mat\'ern model, one can refer to \citet{porcu2024matern}.

\subsection{Univariate Mat\'ern covariance functions}

The standard Mat\'ern covariance function $K_{\nu}$ with smoothness
parameter $\nu>0$ is defined by
\begin{equation}
K_{\nu}(u)=k_{\nu}(\left|u\right|):=\frac{(\sqrt{2\nu}\left|u\right|)^{\nu}}{2^{\nu-1}\Gamma(\nu)}\mathcal{K}_{\nu}(\sqrt{2\nu}\left|u\right|)\,,\,\,u\in\mathbb{R}\label{eq:matern_general}
\end{equation}
where $\mathcal{K}_{\nu}$ is the modified Bessel function \citep[10.25]{DLMF},
see \citet[p. 84]{rasmussen2006gaussian}. In practice, we work with
the scaled Mat\'ern covariance function
\begin{equation}
K_{\nu;\varsigma,\ell}(u):=\varsigma^{2}K_{\nu}(u/\ell)\label{eq:scaled_matern}
\end{equation}
with outputscale parameter $\varsigma>0$ and lengthscale parameter
$\ell>0$ (see also \citet{wang2023matern} for alternative parameterizations
of the Mat\'ern covariance). The modified Bessel function $\mathcal{K}_{\nu}$
does not in general admit an analytical formulation, except in the
case when $\nu$ is a positive half-integer. The first few analytical
modified Bessel functions are given below, where $u>0$:
\begin{align*}
\mathcal{K}_{1/2}(u) & =\sqrt{\frac{\pi}{2}}\frac{e^{-u}}{\sqrt{u}}\\
\mathcal{K}_{3/2}(u) & =\sqrt{\frac{\pi}{2}}\frac{e^{-u}(u+1)}{u^{3/2}}\\
\mathcal{K}_{5/2}(u) & =\sqrt{\frac{\pi}{2}}\frac{e^{-u}(u^{2}+3u+3)}{u^{5/2}}\\
\mathcal{K}_{7/2}(u) & =\sqrt{\frac{\pi}{2}}\frac{e^{-u}(u^{3}+6u^{2}+15u+15)}{u^{7/2}}\\
\mathcal{K}_{9/2}(u) & =\sqrt{\frac{\pi}{2}}\frac{e^{-u}(u^{4}+10u^{3}+45u^{2}+105u+105)}{u^{9/2}}
\end{align*}
Applying these analytical expressions in equation \eqref{eq:matern_general}
and using the Gamma function formula $\Gamma\left(\frac{1}{2}+n\right)=\sqrt{\pi}\frac{(2n)!}{4^{n}n!}$
for $n\in\mathbb{N}$ yields the first few analytical Mat\'ern covariance
functions:
\begin{align}
K_{1/2}(u) = k_{1/2}(\left|u\right|) & =e^{-\left|u\right|}\label{eq:matern_1_2}\\
K_{3/2}(u) = k_{3/2}(\left|u\right|) & =\left(1+\sqrt{3}\left|u\right|\right)e^{-\sqrt{3}\left|u\right|}\label{eq:matern_3_2}\\
K_{5/2}(u) = k_{5/2}(\left|u\right|) & =\left(1+\sqrt{5}\left|u\right|+\frac{5}{3}\left|u\right|^{2}\right)e^{-\sqrt{5}\left|u\right|}\label{eq:matern_5_2}\\
K_{7/2}(u) = k_{7/2}(\left|u\right|) & =\left(1+\sqrt{7}\left|u\right|+\frac{14}{5}\left|u\right|^{2}+\frac{7\sqrt{7}}{15}\left|u\right|^{3}\right)e^{-\sqrt{7}\left|u\right|}\label{eq:matern_7_2}\\
K_{9/2}(u) = k_{9/2}(\left|u\right|) & =\left(1+3\left|u\right|+\frac{27}{7}\left|u\right|^{2}+\frac{18}{7}\left|u\right|^{3}+\frac{27}{35}\left|u\right|^{4}\right)e^{-\sqrt{9}\left|u\right|}\label{eq:matern_9_2}
\end{align}
More generally, the Mat\'ern covariance kernel with $\nu=p+1/2$
where $p\in\mathbb{N}$ is given explicitly by
\begin{equation}
K_{p+1/2}(u)=\left(\sum_{i=0}^{p}\frac{p!}{i!(p-i)!}\frac{(2p-i)!}{(2p)!}\left(2\sqrt{2p+1}\left|u\right|\right)^{\!i}\right)\exp\left(-\sqrt{2p+1}\left|u\right|\right)\label{eq:matern_explicit}
\end{equation}

Remark that equation~\eqref{eq:matern_explicit} converges to the squared exponential
covariance function (Gaussian kernel) when $p\rightarrow\infty$:
\begin{equation}
K_{p+1/2}(u)\underset{p\rightarrow\infty}{\longrightarrow}\exp(-u^{2}/2)\label{eq:matern_limit}
\end{equation}
In practice, the most commonly used particular cases are the Mat\'ern-1/2
 \eqref{eq:matern_1_2}, Mat\'ern-3/2 \eqref{eq:matern_3_2}
and Mat\'ern-5/2 \eqref{eq:matern_5_2} kernels, even
though higher-order Mat\'ern kernels are occasionally used. Remark
that the Mat\'ern-1/2 kernel (exponential covariance function) is
the covariance function of the Ornstein-Uhlenbeck process.

\subsection{Derivative of Mat\'ern covariance functions}

For application purposes, we are often interested in estimating the
two scaling parameters $\varsigma>0$ and $\ell>0$ on a given dataset.
This can be done by gradient descent, which can be implemented either
analytically or by automatic differentiation. For convenience, we
provide the analytical gradient formulas of $K_{\nu;\varsigma,\ell}(u)$
with respect to $\varsigma$ and $\ell$. From equation \eqref{eq:scaled_matern},
the gradient with respect to the outputscale $\varsigma$ is straightforward:
\begin{equation}
\frac{\partial K_{\nu;\varsigma,\ell}(u)}{\partial\sigma}=2\varsigma K_{\nu}(u/\ell)\label{eq:dkernel_dvarsigma}
\end{equation}
for every $\nu>0$, $\varsigma>0$, $\ell>0$. We now turn to the
gradient with respect to the lengthscale $\ell>0$. It is given by
\begin{equation}
\frac{\partial K_{\nu;\varsigma,\ell}(u)}{\partial\ell}=-\frac{\varsigma^{2}}{\ell^{2}}\left|u\right|k_{\nu}^{'}(\left|u\right|/\ell)\label{eq:dkernel_dl}
\end{equation}
which requires to compute the first-order derivative $k_{\nu}^{'}$
of the standard Mat\'ern covariance function. For the first few half-integer
values of $\nu>0$, it is given by:
\begin{align}
k_{1/2}^{'}(u) & =-e^{-u}\label{eq:dmatern12_dl}\\
k_{3/2}^{'}(u) & =-\left(3u\right)e^{-\sqrt{3}u}\label{eq:dmatern32_dl}\\
k_{5/2}^{'}(u) & =-\left(\frac{5}{3}u+\frac{5\sqrt{5}}{3}u^{2}\right)e^{-\sqrt{5}u}\label{eq:dmatern52_dl}\\
k_{7/2}^{'}(u) & =-\left(\frac{7}{5}u+\frac{7\sqrt{7}}{5}u^{2}+\frac{49}{15}u^{3}\right)e^{-\sqrt{7}u}\label{eq:dmatern72_dl}\\
k_{9/2}^{'}(u) & =-\left(\frac{9}{7}u+\frac{27}{7}u^{2}+\frac{162}{35}u^{3}+\frac{81}{35}u^{4}\right)e^{-\sqrt{9}u}\label{eq:dmatern92_dl}
\end{align}
and the general formula for $\nu=p+1/2$ where $p\in\mathbb{N}$ is
given by
\begin{equation}
k_{p+1/2}^{'}(u)=-\left(\sum_{i=1}^{p}\frac{i\sqrt{2p+1}}{(2p-i)}\frac{p!}{i!(p-i)!}\frac{(2p-i)!}{(2p)!}\!\left(2\sqrt{2p+1}\right)^{\!i}u^{i}\!\right)e^{-\sqrt{2p+1}u}\label{eq:dmatern_explicit_dl}
\end{equation}
from which we can deduce the derivative of $K_{\nu;\varsigma,\ell}$
with respect to the lengthscale $\ell$ using equation \eqref{eq:dkernel_dl}. These kernel derivative formulas have been used to obtain the kernel gradient decomposition formulas in Subsection~\ref{subsec:multivariate_matern_decomposition}.

\end{document}